\documentclass{article}

\pdfoutput=1

\usepackage{microtype}
\usepackage{graphicx}
\usepackage{subfigure}
\usepackage{booktabs} %

\usepackage{hyperref}

\usepackage[accepted]{icml2024}

\usepackage{amsmath}
\usepackage{amssymb}
\usepackage{mathtools}
\usepackage{amsthm}

\usepackage{tikz}
\usepackage{pgfplots}
\pgfplotsset{
    compat=1.18,
    every mark/.append style={solid},
}
\usetikzlibrary{plotmarks}

\usepackage[
  separate-uncertainty = true,
  multi-part-units = repeat
]{siunitx}

\theoremstyle{plain}
\newtheorem{theorem}{Theorem}%
\newtheorem{proposition}[theorem]{Proposition}

\theoremstyle{definition}

\newtheorem{assumption}{Assumption}
\theoremstyle{remark}

\DeclareMathOperator*{\argmax}{arg\,max}
\newcommand{\dd}{\mathrm{d}}

\newcommand{\algorithmicnotation}{\textbf{notation}}
\newcommand{\algorithmicinput}{\textbf{input}}
\newcommand{\algorithmicoutput}{\textbf{output}}
\newcommand{\algorithmicfunction}{\textbf{function}}
\newcommand{\algorithmicendfunction}{\textbf{end function}}

\newcommand{\INPUT}{\item[\algorithmicinput]}
\newcommand{\OUTPUT}{\item[\algorithmicoutput]}
\newcommand{\NOTATION}{\item[\algorithmicnotation]}

\newcommand{\FUNCTION}{\item[\algorithmicfunction]}
\newcommand{\ENDFUNCTION}{\item[\algorithmicendfunction]}

\usepackage[]{todonotes}
\usepackage{setspace}
\usepackage{ifthen}

\icmltitlerunning{Nesting Particle Filters for Experimental Design in Dynamical Systems}

\mathtoolsset{showonlyrefs} %

\begin{document}

\setlength{\abovedisplayskip}{5pt}
\setlength{\belowdisplayskip}{5pt}

\setlength{\textfloatsep}{5pt}

\twocolumn[
\icmltitle{Nesting Particle Filters for Experimental Design in Dynamical Systems}

\begin{icmlauthorlist}
\icmlauthor{Sahel Iqbal}{aalto}
\icmlauthor{Adrien Corenflos}{aalto}
\icmlauthor{Simo S\"arkk\"a}{aalto}
\icmlauthor{Hany Abdulsamad}{aalto}
\end{icmlauthorlist}

\icmlaffiliation{aalto}{Aalto University, Espoo, Finland}

\icmlcorrespondingauthor{Sahel Iqbal}{sahel.iqbal@aalto.fi}

\icmlkeywords{Bayesian experimental design, sequential Monte Carlo, non-exchangeable data, amortized inference}

\vskip 0.3in
]

\printAffiliationsAndNotice{}  %

\begin{abstract}
    In this paper, we propose a novel approach to Bayesian experimental design for non-exchangeable data that formulates it as risk-sensitive policy optimization. We develop the Inside-Out SMC\textsuperscript{2} algorithm, a nested sequential Monte Carlo technique to infer optimal designs, and embed it into a particle Markov chain Monte Carlo framework to perform gradient-based policy amortization. Our approach is distinct from other amortized experimental design techniques, as it does not rely on contrastive estimators. Numerical validation on a set of dynamical systems showcases the efficacy of our method in comparison to other state-of-the-art strategies.
\end{abstract}

\section{Introduction}
Traditionally, Bayesian inference on the parameters of a statistical model is performed \emph{after the fact} by employing a posterior elicitation routine to previously gathered data. However, in many scenarios, experimenters can proactively \emph{design} experiments to acquire maximal information about the parameters of interest. Bayesian experimental design~\citep[BED,][]{lindley1956measure, chaloner1995bayesian} offers a principled framework to achieve this goal by maximizing the expected information gain over the design space. BED has found applications in fields as varied as active learning~\citep{bickfordsmith2023prediction}, neuroscience~\citep{shababo2013bayesian}, physics~\citep{mcmichael2021sequential}, psychology~\citep{myung2013tutorial}, and robotics~\citep{schultheis2020receding}. A recent overview of the field of BED can be found in \citet{rainforth2024modern}.

In Bayesian experimental design, we are given a prior $p(\theta)$ and a likelihood $p(x \mid \xi, \theta)$, where $\theta \in \Theta$ is the set of unknown parameters of interest, $x \in \mathcal{X}$ is the experimental outcome, and $\xi \in \Xi$ is a controllable design. The information gain~\citep[IG,][]{lindley1956measure} in a parameter $\theta$ upon applying a design $\xi$ and observing an outcome $x$ is defined as
\begin{equation}\label{eq:ig_single_experiment}
    \mathcal{G}(x, \xi) \coloneq \mathbb{H}[p(\theta)] - \mathbb{H}[p(\theta\mid x, \xi)],
\end{equation}
with $\mathbb{H}[p(\cdot)] \coloneq - \mathbb{E}_{p(\cdot)}[\log p(\cdot)]$ denoting the entropy of a random variable with probability density $p$. Since the outcomes $x$ are themselves random variables for a fixed design $\xi$, the goal in BED is to choose a design $\xi^*$ that maximizes the \emph{expected} information gain~(EIG), defined as
\begin{equation}\label{eq:eig_single_experiment}
    \mathcal{I}(\xi) \coloneq \mathbb{E}_{p(x \mid \xi)}\bigl[\mathbb{H}[p(\theta)] - \mathbb{H}[p(\theta\mid x, \xi)]\bigr],
\end{equation}
where $p(x \mid \xi) = \mathbb{E}_{p(\theta)}[p(x \mid \theta, \xi)]$. The expected information gain thus quantifies the decrease in uncertainty in the unknown variable $\theta$ upon implementing a design $\xi$.

While mathematically elegant, the BED framework involves maximizing nested expectations over intractable quantities such as the marginal likelihood of $x$ and the posterior of $\theta$ appearing in~\eqref{eq:eig_single_experiment}. This is a computationally intensive task~\citep{kueck2009inference, rainforth2018nesting}, which becomes even more challenging when optimizing designs for a series of experiments conducted sequentially, where the impact of each individual design needs to be accounted for across the entire sequence of experiments. This makes the deployment of sequential BED close to impossible on real-time systems with high-frequency data.

\citet{huan2016sequential} addressed this limitation by introducing a parametric policy to predict designs as a function of the running parameter posterior, thereby eliminating the cost of the maximization step in each experiment. \citet{foster2021deep} extended this idea to \textit{amortize} the overall cost of sequential BED by conditioning the policy on the entire outcome-design history and avoiding explicit posterior computation. In that approach, called \textit{Deep Adaptive Design}~(DAD), there is an upfront cost to learning the policy, but experiments can be performed in real-time. While DAD is only applicable for exchangeable models, \textit{implicit DAD}~\citep[iDAD,][]{ivanova2021implicit} generalized the concept to accommodate non-exchangeable models that cover dynamical systems. These methods leverage a lower bound to the EIG known as the sequential Prior Contrastive Estimation~(sPCE) bound. sPCE exhibits significant bias in low-sample regimes, thus requiring a large number of samples for accurate estimates of the EIG~\citep{foster2021deep}. 

In this work, we introduce a novel amortization scheme that circumvents the drawbacks of sPCE. Our approach is rooted in understanding sequential Bayesian experimental design as an adaptive risk-sensitive decision-making process~\citep{whittle1990risk}. We demonstrate that risk-sensitive decision-making can be cast as an inference problem for an equivalent non-Markovian non-linear and non-Gaussian state-space model \cite{toussaint2006probabilistic, rawlik2013probabilistic}. This insight motivates a novel nested sequential Monte Carlo (SMC) algorithm that jointly estimates the EIG and the corresponding optimal designs. We refer to this algorithm as \textit{Inside-Out SMC\textsuperscript{2}}, due to its relation to the SMC$^2$ algorithm of \citet{chopin2013smc}. Finally, we embed our technique within a particle Markov chain Monte Carlo~(pMCMC) algorithm to perform gradient-based optimization of the amortizing policy.

We validate our algorithm on a range of dynamical systems with long experiment sequences, highlighting the computational advantages of our proposed method compared to existing work. The code to reproduce our results is available at \href{https://github.com/Sahel13/InsideOutSMC2.jl}{https://github.com/Sahel13/InsideOutSMC2.jl}.%

\section{Problem Statement}\label{sec:problem_statement}

We are interested in the sequential BED problem for non-exchangeable data, specifically dynamical systems. Accordingly, we assume a scenario of $\, T \in \mathbb{N}$ sequential experiments to infer a parameter vector $\theta$, starting from a prior $p(\theta)$, a Markovian likelihood $f(x_{t+1} \mid x_{t}, \xi_{t}, \theta)$, and an initial distribution $p(x_{0})$, where $t$ indexes the experiment number. We further assume that the designs are sampled from a \textit{stochastic} policy $\pi_\phi(\xi_t \mid z_{0:t})$ parameterized by $\phi$, where we define $z_0 \coloneq \{x_0\}$ and $z_{t} \coloneq \{x_t, \xi_{t-1}\}$ for all $t \geq 1$, and denote the outcome-design history up to time $t$ by $z_{0:t} \coloneq \{x_{0:t}, \xi_{0:t-1}\}$. This yields the following factorization for the joint distribution of outcomes and designs:
\begin{align}\label{eq:joint_density}
    p_{\phi}(z_{0:T} \mid \theta) & = p(z_{0}) \prod_{t=1}^T p_{\phi}(z_{t} \mid z_{0:t-1}, \theta) \\
    & = \nonumber
    \begin{aligned}[t]
        p(x_{0}) & \left\{ \prod_{t=1}^T f(x_t \mid x_{t-1}, \xi_{t-1}, \theta) \right\} \\
        & \times \left\{ \prod_{t=1}^{T} \pi_\phi(\xi_{t-1} \mid z_{0:t-1}) \right\}.
    \end{aligned}
\end{align}

In this setting, the expected information gain can be written analogously to that of the single experiment:
\begin{align}\label{eq:sequential_eig}
    \mathcal{I}(\phi) &\coloneq \mathbb{E}_{p_{\phi}(z_{0:T})} \Bigl[ \mathbb{H}\bigl[p(\theta)\bigr] - \mathbb{H}\bigl[p(\theta \mid z_{0:T})\bigr] \Bigr],
\end{align}
where $p_{\phi}(z_{0:T}) = \mathbb{E}_{p(\theta)} \bigl[p_{\phi}(z_{0:T} \mid \theta)\bigr]$. This definition corresponds to the \textit{terminal reward} framework in literature~\citep[Section 1.8]{foster2021thesis}, as it compares the prior entropy of $\theta$ with the posterior entropy at the end of the experiment sequence. Note that the EIG in~\eqref{eq:sequential_eig} is evaluated under the expectation of the marginal distribution $p_\phi(z_{0:T})$, including the stochastic design policy. The resulting experimental design objective corresponds to finding the optimal policy parameters $\phi^* \coloneq \argmax_\phi \,\, \mathcal{I}(\phi)$. The upcoming sections will present a novel interpretation of this objective in a sequential decision-making framework that leverages the duality with inference techniques to perform policy amortization.

\section{Sequential Bayesian Experimental Design as Probabilistic Inference}\label{sec:bed_as_inference}
To formulate sequential BED as an inference problem, we demonstrate a factorization of the EIG over time steps.
\begin{proposition}\label{prop:eig_factorization}
    For models specified by the joint density in~\eqref{eq:joint_density}, the expected information gain factorizes to
    \begin{equation}\label{eq:eig_factorization}
        \mathcal{I}(\phi) = \mathbb{E}_{p_{\phi}(z_{0:T})} \biggl[ \sum_{t=1}^T r_{t}(z_{0:t}) \biggr],
    \end{equation}
    where $r_{t}(z_{0:t})$ is a stage reward defined as 
    \begin{equation}\label{eq:stage_reward}
        r_{t}(z_{0:t}) = \alpha_{t}(z_{0:t}) + \beta_{t}(z_{0:t}),
    \end{equation}
    with $\alpha_{t}(z_{0:t})$ and $\beta_{t}(z_{0:t})$ defined as
    \begin{align*}
        \!\!\! \alpha_{t}(z_{0:t}) &= \int p(\theta \mid z_{0:t}) \log f(x_t \mid x_{t-1}, \xi_{t-1}, \theta) \, \dd \theta, \\
        \!\!\! \beta_{t}(z_{0:t}) & \! = - \log \! \int \! p(\theta \mid z_{0:t-1}) f(x_t \mid x_{t-1}, \xi_{t-1}, \theta) \, \dd \theta.
    \end{align*}
    Furthermore, for models with additive, constant noise in the dynamics, the EIG can be written as
    \begin{equation}\label{eq:eig_constant_noise}
        \mathcal{I}(\phi) \equiv \mathbb{E}_{p_{\phi}(z_{0:T})} \left[ \sum_{t=1}^T \beta_{t}(z_{0:t}) \right],
    \end{equation}
    where '$\equiv$' denotes equality up to an additive constant.
\end{proposition}
The proof is given in Appendix~\ref{app:prop_1_proof}. Written in this form, the expected information gain resembles the expected total reward of a discrete-time, finite-horizon, non-Markovian decision-making problem~\citep{puterman2014markov} with a stage reward $r_{t}(z_{0:t})$ that captures the information content regarding the unknown parameters $\theta$. We will now use this factorization of $\mathcal{I}(\phi)$ to derive a risk-sensitive objective and a dual inference perspective, leading to a novel amortized BED learning scheme.

\subsection{The Dual Inference Problem}
To leverage the duality between risk-sensitive decision-making and inference, we follow the formulation of \citet{toussaint2006probabilistic} and \citet{rawlik2013probabilistic}, and introduce the potential function
\begin{equation}\label{eq:potential-function}
    g_{t}(z_{0:t}) \coloneq \exp \Big\{ \eta \, r_{t}(z_{0:t}) \Big\},
\end{equation}
with $\eta \in \mathbb{R}_{>0}$. If we define the potential of an entire trajectory, $g_{1:T}(z_{0:T})$, to be the product of the potential functions over time steps, then
\begin{equation}
    \log g_{1:T}(z_{0:T}) = \sum_{t=1}^T \log g_{t}(z_{0:t}) \! = \eta \left[ \sum_{t=1}^T r_{t}(z_{0:t}) \right],
\end{equation}
is the total reward of a trajectory scaled by $\eta$. In this context, the potentials $g_{1:T}$ play the role of an un-normalized pseudo-likelihood proportional to the probability of the trajectory $z_{0:T}$ being optimal~\citep{dayan1997using}. This perspective allows us to define a non-Markovian state-space model characterized, for $t=0, \ldots, T$, by the following joint density
\begin{equation}\label{eq:pathwise_smoothing_trajectory}
    \!\!\!\! \Gamma_t(z_{0:t}; \phi) = \frac{1}{Z_t(\phi)} \, p(z_0) \prod_{s=1}^t p_\phi(z_s \mid z_{0:s-1}) \, g_s(z_{0:s}),\!\!
\end{equation}
where $p_\phi(z_t \mid z_{0:t-1})$ are the marginal dynamics under the running filtered posterior $p(\theta \mid z_{0:t-1})$
\begin{equation}\label{eq:marginal_dynamics}
    \!\!\! p_\phi(z_t \mid z_{0:t-1}) \!=\! \int \! p_\phi(z_t \mid z_{0:t-1}, \theta) \, p(\theta \mid z_{0:t-1}) \, \dd \theta,
\end{equation}
and $Z_t(\phi)$ is the normalizing constant
\begin{equation}
    Z_t(\phi) = \int g_{1:t}(z_{0:t}) \, p_\phi(z_{0:t}) \, \dd z_{0:t}.
\end{equation}
For ease of exposition, we will henceforth refer to $Z_T(\phi)$ as the marginal likelihood of \emph{being optimal}, even though it may not represent any meaningful probability.

The duality principle becomes evident when we apply Jensen's inequality to show that the log marginal likelihood is an upper bound on the EIG scaled by $\eta$:
\begin{subequations}\label{eq:log_marginal_jensen}
    \begin{align}
        \log Z_T(\phi) &= \log \mathbb{E}_{p_\phi(z_{0:T})} \left[g_{1:T}(z_{0:T})\right] \\
        &\geq \mathbb{E}_{p_\phi(z_{0:T})} \left[\log g_{1:T}(z_{0:T})\right] \\
        &= \eta \, \mathcal{I}(\phi).
    \end{align}
\end{subequations}
Hence, maximizing the marginal likelihood is equivalent to maximizing a risk-sensitive EIG objective, which we denote as $\mathcal{I}_{\eta}(\phi)$~\citep{marcus1997, rawlik2013probabilistic},
\begin{equation}\label{eq:risk_sensitive_eig}
     \mathcal{I}_{\eta}(\phi) = \frac{1}{\eta} \log \mathbb{E}_{p_\phi(z_{0:T})} \left[\exp \left\{\eta \sum_{t=1}^T r_t(z_{0:t})\right\}\right].
\end{equation}
Note that $\eta$ modulates the bias-variance trade-off of this objective. This aspect is revealed by considering a first-order expansion of the objective around $\eta = 0$
\begin{equation}
    \mathcal{I}_{\eta}(\phi) \approx \mathbb{E} \left[\sum_{t=1}^T r_t(z_{0:t})\right] + \frac{\eta}{2} \mathbb{V} \left[ \sum_{t=1}^T r_t(z_{0:t}) \right],
\end{equation}
where the expectation and variance operators are in relation to $p_\phi(z_{0:T})$. Note that in the limit $\eta \rightarrow 0$, we recover the risk-neutral EIG objective from \eqref{eq:eig_factorization}. 

The choice of a positive tempering parameter $\eta \in \mathbb{R}_{>0}$ leads to a risk-seeking objective that incentivizes exploration during policy amortization. This is compatible with the heuristic of \emph{optimism in the face of uncertainty}, widely adopted in stochastic optimization settings~\citep{neu2020unifying}.

In this section, we formalized the connection between a risk-sensitive sequential BED objective and inference in an equivalent non-Markovian state-space model. In the following, we leverage this insight to formulate a gradient-based policy optimization technique within a particle MCMC framework.

\subsection{Amortization as Likelihood Maximization}\label{sec:mle}
The duality principle demonstrated in Section~\ref{sec:bed_as_inference} enables us to view amortized BED from an inference-centric perspective and to frame policy optimization in terms of maximum likelihood estimation~(MLE) in a non-Markovian, nonlinear, and non-Gaussian state-space model, as specified by~\eqref{eq:pathwise_smoothing_trajectory}. Following the literature on particle methods for MLE~\citep{kantas2015particle}, we employ a stochastic gradient ascent algorithm~\citep{robbins1951stochastic}.

To obtain the derivative of the log marginal likelihood, $\mathcal{S}(\phi) \coloneq \nabla_\phi \log Z_T(\phi)$, also known as the score function, we make use of Fisher's identity~\citep{cappe2005inference},
\begin{align*}
     \mathcal{S}(\phi) &= \int \nabla_\phi \log \tilde{\Gamma}_T(z_{0:T}; \phi) \, \Gamma_T(z_{0:T}; \phi) \, \dd z_{0:T} \\
    &= \int \nabla_\phi \log p_\phi(z_{0:T}) \, \Gamma_T(z_{0:T}; \phi) \, \dd z_{0:T}.
\end{align*}
where we define $\tilde{\Gamma}_T(z_{0:T}; \phi) \coloneq p_\phi(z_{0:T}) \, g_{1:T}(z_{0:T})$ to be the un-normalized density from~\eqref{eq:pathwise_smoothing_trajectory}. This identity provides a Monte Carlo estimate of the score $\hat{\mathcal{S}}(\phi)$ given samples from $\Gamma_T(\cdot; \phi)$. It is well-known that computing this expectation naively by first sampling from $p_\phi(z_{0:T})$ and then weighting the samples by $g$ results in very high variance estimates~\citep[see, e.g., in][Section 3.3]{doucet2009tutorial}. Alternatively, a lower variance estimate can be achieved by drawing approximate samples from $\Gamma_T(\cdot; \phi)$ via particle smoothing, yielding consistent, albeit biased, estimates of expectations under the smoothing distribution for a finite sample size~\citep[Chapter 12]{chopin2020introduction}.

Another alternative is to use Markovian score climbing~\citep[MSC,][]{gu1998stochastic, naesseth2020}, see Algorithm~\ref{alg:msc}. MSC uses a $\Gamma_T(\cdot; \phi)$-ergodic Markov chain Monte Carlo \citep[MCMC, see, e.g.,][for a review and definition]{brooks2011handbook} kernel, $\mathcal{K}_{\phi}(\cdot \mid z_{0:T})$, to compute a Monte Carlo estimate of the score. Contrary to simply using particle smoother approximations within a gradient ascent procedure~\citep[see][Section 5]{kantas2015particle}, Algorithm~\ref{alg:msc} is guaranteed to converge to a local optimum of the marginal likelihood~\citep[Proposition 1]{naesseth2020}.

\begin{algorithm}[t]
    \setstretch{1.1}
    \caption{Markovian score climbing}
    \label{alg:msc}
    \begin{algorithmic}[1]
        \INPUT{Initial trajectory $z_{0:T}^0$, initial parameters $\phi_0$, step size sequence $\{\gamma_i\}_{i \in \mathbb{N}}$, Markov kernel $\mathcal{K}$.}
        \OUTPUT{Local optimum $\phi^*$ of the marginal likelihood.}
        \STATE $k \gets 1$
        \WHILE{not converged}
            \STATE Sample $z_{0:T}^k \sim \mathcal{K}_{\phi_{k-1}}(\cdot \mid z_{0:T}^{k-1})$ 
            \STATE Compute $\hat{\mathcal{S}}(\phi_{k-1}) \gets \nabla_\phi \log p_\phi(z^{k}_{0:T}) \vert_{\phi=\phi_{k-1}}$ 
            \STATE Update $\phi_k \gets \phi_{k-1} + \gamma_k \, \hat{\mathcal{S}}(\phi_{k-1})$
            \STATE $k \gets k + 1$
        \ENDWHILE
        \STATE \textbf{return} $\phi_k$
    \end{algorithmic}
\end{algorithm}

In this work, we construct the MCMC kernel $\mathcal{K}_\phi$ as a variant of the conditional sequential Monte Carlo~(CSMC) kernel~\citep{andrieu2010particle}, namely the Rao--Blackwellized CSMC kernel~\citep{Olsson2017Paris, cardoso2023STATE, abdulsamad2023risk}. 
In practice, CSMC (as well as its Rao--Blackwellized modifications aforementioned) can be implemented as a simple modification to a particle filter representation of the smoothing distribution~\citep[Section 4.1]{kitagawa1996montecarlo}. 
In the next section, we describe how such a particle filter can be implemented for~\eqref{eq:pathwise_smoothing_trajectory}, and, for the sake of brevity, we defer the full description of the conditional version to Appendix~\ref{sec:csmc}.

\section{Inside-Out \texorpdfstring{SMC\textsuperscript{2}}{SMC²}}\label{sec:io_smc}

\subsection{Approximating the Filtered Posterior}\label{subsec:param}
A \emph{bootstrap} particle filter samples particles from the transition density and weights them using the potential function~\citep[][Section 10.3]{chopin2020introduction}. In our non-Markovian model, this would imply sampling from the marginal dynamics $p_\phi(z_t \mid z_{0:t-1})$ in \eqref{eq:marginal_dynamics}, and evaluating the potential function $g_t(z_{0:t})$ in \eqref{eq:potential-function}. Both steps require computing the filtered posterior $p(\theta \mid z_{0:t-1})$. Fortunately, for models of the form
\begin{equation*}
    \theta \sim p(\theta), \quad z_t \sim p_\phi(z_t \mid z_{0:t-1}, \theta), \quad t \geq 1,
\end{equation*}
the iterated batch importance sampling (IBIS) algorithm of \citet{chopin2002sequential} can be used to generate weighted Monte Carlo samples $\{\theta^m_t\}_{m=1}^M \eqqcolon \theta^{1:M}_t$ that are approximately distributed according to $p(\theta \mid z_{0:t})$ at each time step using a specialized particle filtering procedure. We summarize a single step of the method in Algorithm~\ref{alg:ibis}.

\begin{algorithm}[t]
    \caption{Single step of IBIS}
    \label{alg:ibis}
    \begin{algorithmic}[1]
    \NOTATION Any operation with superscript $m$ is to be understood as performed for all $m = 1, \dots, M$.
    \FUNCTION{\textsc{Ibis\_Step}}$(z_{0:t}, \theta^{1:M}, W^{1:M})$
        \STATE Compute $v_t(\theta^m) = p_\phi(z_t \mid z_{0:t-1}, \theta^m)$.
        \STATE Reweight: $W^m \propto W^m v_t(\theta^m)$. \label{line:reweight-ibis}
        \IF{some degeneracy criterion is fulfilled}
            \STATE Resample: $a_t^m \sim \mathcal{M}(W^{1:M})$. \label{ibis:resample}
            \STATE Move: $\tilde{\theta}^m \sim Q_t(\theta^{a_t^m}, \cdot)$. \label{ibis:move}
            \STATE Replace the current set of weighted particles with
            $(\theta^m, W^m) \gets (\tilde{\theta}^m, 1 / M)$. \label{ibis:replace}
        \ENDIF
        \STATE \textbf{return} $\{\theta^m, W^m\}_{m=1}^M$
    \ENDFUNCTION
    \end{algorithmic}
\end{algorithm}

In Algorithm~\ref{alg:ibis}, $\mathcal{M}(W^{1:M})$ denotes multinomial sampling using the normalized weights $W^{1:M}$, and $Q_t$ is a $p(\theta \mid z_{0:t})$-ergodic Markov chain Monte Carlo kernel. If a degeneracy criterion is met, IBIS employs a resample-move step~\citep[lines~\ref{ibis:resample}--\ref{ibis:replace} in Algorithm~\ref{alg:ibis}, see][]{gilks2001following} to rejuvenate particles using the Markov kernel $Q_t$. A standard degeneracy measure, which we use in this work, is given by the effective sample size~(ESS)~\citep{liu1995blind} of the particle representation computed as $\textrm{ESS} = 1/\sum_{m=1}^M (W^m)^2$. The ESS roughly corresponds to the number of equivalent independent samples one would need to compute integrals with the same precision. The resample-move step is then triggered if the ESS falls below a chosen fraction of the total particles $M$, taken to be $75\%$ in this work. Details on the choice of the Markov kernel $Q_t$ are given in Appendix~\ref{sec:ibis_kernel}.

\subsection{The Inside-Out \texorpdfstring{SMC\textsuperscript{2}}{SMC²} Algorithm}\label{subsec:algorithm}
\citet{chopin2004central} showed that, for integrable functions $\psi$, $\sum_{m=1}^M W^m \, \psi(\theta^m)$ is a consistent and asymptotically (as $M \to \infty$) normal estimator of the integral $\int \psi(\theta) \, p(\theta \mid z_{0:t}) \, \dd z_{0:t}$.
Therefore, a natural solution to perform inference in~\eqref{eq:pathwise_smoothing_trajectory} is to use IBIS within a standard particle filter. This idea is similar to the SMC\textsuperscript{2} algorithm of \citet{chopin2013smc} which can be seen as a particle filter within IBIS targeting the distribution $p(\theta \mid y_{1:t})$ for a state-space model with noisy observations $y_{1:t}$. We thus call our algorithm \textit{Inside-Out SMC$^2$}, which we reproduce in Algorithm~\ref{alg:i-o-smc2}, with $N$ and $M$ denoting the numbers of samples of $z$ and $\theta$ respectively. Any operation therein with superscripts $m$ or $n$ is to be understood as performed for every $m = 1, \dots, M$ and $n = 1, \dots, N$, and we use an upper script dot ${u}^{\bullet n}$ to denote the collection $\{{u}^{m n}\}_{m=1}^M$.

At a time step $t$ and given a trajectory $z_{0:t}^n$, IBIS approximates the $\theta$-posterior with weighted samples $\{\theta_t^{\bullet n}, W^{\bullet n}_{t, \theta}\}$ $$p(\theta \mid z_{0:t}^n) \approx \hat{p}(\theta \mid z_{0:t}^n) \coloneq \sum_{m=1}^M W_{t,\theta}^{mn} \, \delta_{\theta_t^{mn}}(\theta),$$ where $\delta$ is the Dirac delta function. We can then form an approximation to the marginal dynamics as follows
\begin{align*}
    \hat{p}(x_{t+1} \mid z_{0:t}^n, \xi_t) &= \int f(x_{t+1} \mid x_t^n, \xi_t, \theta) \, \hat{p}(\theta \mid z_{0:t}^n) \, \dd\theta \\
    &= \sum_{m=1}^M W_{t,\theta}^{mn} \, f(x_{t+1} \mid x_t^n, \xi_t, \theta_t^{mn}),
\end{align*}
and consequently, for the augmented state as
\begin{equation*}
    \hat{p}_\phi(z_{t+1} \mid z_{0:t}^n) = \hat{p}(x_{t+1} \mid z_{0:t}^n, \xi_t) \, \pi_\phi(\xi_t \mid z_{0:t}^n).
\end{equation*}

If the Markovian density $f(x_t \mid x_{t-1}, \xi_{t-1}, \theta)$ is conditionally linear in the parameters and conjugate to the prior $p(\theta)$, we can compute the posterior in closed form, and therefore the marginal dynamics as well. In this case, we do not need IBIS, and this significantly reduces the computational complexity of Algorithm~\ref{alg:i-o-smc2}~(see Appendix~\ref{sec:conditionally_linear}).

Note that, for the sake of clarity, in Algorithm~\ref{alg:i-o-smc2} and in its analysis in Section~\ref{sec:target_distribution}, we consider the case when the dynamics have constant noise, corresponding to the stage reward in~\eqref{eq:eig_constant_noise}. The modification to Algorithm~\ref{alg:i-o-smc2} for the more general case is straightforward, only requiring a modified weight function, as detailed in Section~\ref{sec:general_weight_function}.

\begin{algorithm}[t]
    \setstretch{1.12}
    \caption{Inside-Out SMC\textsuperscript{2}}
    \label{alg:i-o-smc2}
    \begin{algorithmic}[1]
        \STATE Sample $z_0^n \sim p(z_0)$, $\theta_0^{mn} \sim p(\theta)$, set $W_{0, \theta}^{mn} \gets 1/M$.
        \STATE Sample $z_1^n \sim \hat{p}_\phi(\cdot \mid z_0^n)$ initialize the state history $z_{0:1}^n \gets (z_0^n, z_1^n)$.
        \STATE Compute and normalize the weights
            \vspace{-0.2cm} $$\smash{W_z^n \propto \exp \Big \{- \eta \log \hat{p}(x_1^n \mid x_0^n, \xi_0^n) \Big\}.}$$ \vspace{-0.6cm}
        \FOR{$t \gets 1, \dots, T - 1$}
            \STATE Sample $b_t^n \sim \mathcal{M}(W_z^{1:N})$. \label{line:resampling}
            \STATE $\theta_t^{\bullet n}, W^{\bullet n}_{t, \theta} \gets \textsc{Ibis\_Step}(z_{0:t}^{b_t^n}, \theta_{t-1}^{\bullet b_t^n}, W^{\bullet b_t^n}_{t - 1, \theta})$ \label{line:ibis}
            \label{line:mixture_distribution}
            \STATE Sample
                $z_{t+1}^n \sim \hat{p}_\phi(\cdot \mid z_{0:t}^{b_t^n}),$
                and append to state history
                $z_{0:t+1}^n \gets [z_{0:t}^{b_t^n}, z_{t+1}^n]$. \label{line:marg_sample_2}
            \STATE Compute and normalize the weights
                \vspace{-0.2cm} $$\smash{W_z^n \propto \exp\left\{- \eta \log \hat{p}(x_{t+1}^n \mid z_{0:t}^{b_t^n}, \xi_t^n)\right\}.}$$ \label{line:reweight_2} \vspace{-0.6cm}\label{line:smc2-reweight}
        \ENDFOR
        \STATE \textbf{return} $\{z_{0:T}^n, W_z^n\}_{n=1}^N$.
    \end{algorithmic}
\end{algorithm}
\vspace{-0.5em}

\subsection{Target Distribution of Inside-Out SMC\textsuperscript{2}}\label{sec:target_distribution}
We now show that the nested particle filter introduced in the previous section asymptotically targets the correct distribution. Similarly to \citet[Proposition 1]{chopin2013smc}, Algorithm~\ref{alg:i-o-smc2} is a particle filter targeting a particle filter. Indeed, dropping the $n$ indices and the explicit dependence on $\phi$, let $\Gamma_t^M(z_{0:t}, \theta_{0:t-1}^{1:M}, a_{1:t-1}^{1:M})$ denote the distribution of all stochastic variables generated by an instance of the inner IBIS at line~\ref{line:ibis} in Algorithm~\ref{alg:i-o-smc2}. We first note that
\begin{equation*}
  \Gamma_0^M(z_0) = p(z_0), \quad \Gamma^M_0(z_0, \theta_0^{1:M}) = p(z_0) \prod_{m=1}^M p(\theta_0^m).
\end{equation*}
Let us break down the ratio of the distributions over successive iterations as
\begin{multline}\label{eq:smc2-distribution-ratio}
    \frac{\Gamma^M_{t+1}(z_{0:t+1}, \theta_{0:t}^{1:M}, a_{1:t}^{1:M})}{\Gamma^M_t(z_{0:t}, \theta_{0:t-1}^{1:M}, a_{1:t-1}^{1:M})} = \frac{\Gamma^M_{t+1}(z_{0:t+1}, \theta_{0:t}^{1:M}, a_{1:t}^{1:M})}{\Gamma^M_{t}(z_{0:t}, \theta_{0:t}^{1:M}, a_{1:t}^{1:M})} \\ 
        \times \frac{\Gamma^M_{t}(z_{0:t}, \theta_{0:t}^{1:M}, a_{1:t}^{1:M})}{\Gamma^M_t(z_{0:t}, \theta_{0:t-1}^{1:M}, a_{1:t-1}^{1:M})}.
\end{multline}
The second fraction in~\eqref{eq:smc2-distribution-ratio} the IBIS rejuvenation step:
\begin{equation}\label{eq:ibis}
  \!\!\!\!\frac{\Gamma_{t}^M(z_{0:t}, \theta_{0:t}^{1:M}, a_{1:t}^{1:M})}{\Gamma^M_t(z_{0:t}, \theta_{0:t-1}^{1:M}, a_{1:t-1}^{1:M})} = \prod_{m=1}^M W_{t,\theta}^{a_t^m} \,Q_t(\theta_{t-1}^{a_t^m}, \theta_t^m),
\end{equation}
where the normalized weights (of the $\theta$ particles) $W_{t,\theta}^m$ are %
\begin{equation*}
  W_{t,\theta}^m = \frac{v_t^m}{\sum_{m=1}^M v_t^m}, \quad v_t^m = p_\phi(z_t \mid z_{0:t-1}, \theta_{t-1}^m),
\end{equation*}
and we have assumed multinomial resampling at every time step. The first fraction in~\eqref{eq:smc2-distribution-ratio} is the trajectory update step:
\begin{multline}\label{eq:prediction}
  \!\!\frac{\Gamma^M_{t+1}(z_{0:t+1}, \theta_{0:t}^{1:M}, a_{1:t}^{1:M})}{\Gamma^M_{t}(z_{0:t}, \theta_{0:t}^{1:M}, a_{1:t}^{1:M})} \propto \frac{1}{M} \sum_{m=1}^M p_\phi(z_{t+1} \mid z_{0:t}, \theta_{t}^m) \\ \times \exp\left\{-\eta \log \frac{1}{M} \sum_{m=1}^M f(x_{t+1} \mid x_t, \xi_t, \theta_{t}^m)\right\}.
\end{multline}

The following proposition, akin to the law of large numbers and proven in Appendix~\ref{app:consistency}, ensures that Algorithm~\ref{alg:i-o-smc2} asymptotically targets the correct distribution.

\begin{proposition}[Consistency of the target distribution]\label{prop:consistency}
    Let $\tilde{\Gamma}^M_t(z_{0:t}, \theta_t) = \mathbb{E}\left[\Gamma_t^M(z_{0:t}, \theta_{t}^{1:M}, a_{1:t}^{1:M}) \right]$ be the joint expected empirical distribution\footnote{Note that $\Gamma_t^M(z_{0:t}, \theta_{t}^{1:M}, a_{1:t}^{1:M})$ is a random measure and hence this expectation is a measure.} over $(z_{0:t}, \theta_t)$ taken by integrating over $a_{1:t}^{1:M}$ and $\theta_{0:t}^{1:M}$. Under technical conditions listed in Appendix~\ref{app:consistency}, as $M \to \infty$, empirical expectations under $\tilde{\Gamma}^M_t$ converge to expectations under $\Gamma_t(z_{0:t}, \theta_t) \coloneq \Gamma_t(z_{0:t}) \, p(\theta_t \mid z_{0:t})$. That is, for any bounded test function $\psi(z_{0:t}, \theta_t)$, we have
    \begin{equation}
        \mathbb{E}_{\tilde{\Gamma}^M_t}\left[\psi(z_{0:t}, \theta_t)\right] \to \mathbb{E}_{\Gamma_t}\left[\psi(z_{0:t}, \theta_t)\right]
    \end{equation}
    almost surely.
\end{proposition}

\section{Related Work}
\citet{foster2021deep} propose optimizing the following \textit{sequential Prior Contrastive Estimation}~(sPCE) lower bound to the expected information gain
\begin{equation}\label{eq:spce-bound}
    \mathcal{L}_T^{\textrm{sPCE}}(\phi, L) = \mathbb{E}_{p_{\phi}(\theta_0, z_{0:T})p(\theta_{1:L})}\bigl[ g_L(\theta_{0:L}, z_{0:T})\bigr],
\end{equation}
where
\begin{equation}\label{eq:contrastive_estimator}
    g_L(\theta_{0:L}, z_{0:T}) = \log \frac{p_{\phi}(z_{0:T} \mid \theta_0)}{\frac{1}{L + 1}\sum_{\ell=0}^L p_{\phi}(z_{0:T} \mid \theta_\ell)}.
\end{equation}
A naive nested Monte Carlo estimator would exclude $\theta_0$ from the denominator, but such an estimator would have a large variance, particularly for long experiment sequences~($T \gg 1$). By including $\theta_0$ in the estimate for $p_\phi(z_{0:T})$, the authors show that $g_L$ is upper bounded by $\log(L+1)$, where $L$ is the number of regularizing or \emph{contrastive} samples, resulting in a low variance estimator amenable to optimization~\citep{poole2019variational}. However, this bound also implies that one needs to take an exponentially large value of $L$ to avoid being restricted by the upper bound which, depending on the true value of the EIG, may introduce significant bias.
Our Inside-Out SMC\textsuperscript{2} algorithm does not suffer from this particular drawback and only requires a small number of $\theta$ particles to provide an estimate of the EIG. This is achieved by leveraging the sequential structure of the experiments to maintain a running posterior over $\theta$ instead of relying on samples from the prior.

\citet{drovandi2013sequential}, \citet{drovandi2014sequential}, and \citet{moffat2020sequential} have previously proposed the use of IBIS for experimental design, albeit in the context of exchangeable data. In these approaches, IBIS is used solely to track the parameter posterior given the past experiments. Most importantly, their experiment design is myopic, only optimizing for the next experiment. In contrast, we embed IBIS into a general particle smoothing algorithm, where IBIS is used to approximate the marginalized dynamics. This formulation enables our technique to optimize across the entire horizon of experiments and lends itself to amortization.

Our work also shares similarities in formulation with that of \citet{blau2022optimizing} which formulates the sequential BED problem as a hidden parameter Markov decision process~\citep{doshi-velez2016hidden}. Unlike our work, they adopt a reinforcement learning approach that applies only to exchangeable models and optimize the sPCE bound.

In the context of SMC, a related approach to ours is given by~\citet{wigren2019parameter}, who also propagate the filtering posterior for the parameter at hand to perform parameter-marginalized inference. However, they are not interested in experimental design, and in contrast to our Inside-Out SMC\textsuperscript{2}, they only need to compute the \emph{marginal} posterior distribution, whereas we reuse the \emph{pathwise} posterior distribution to compute our potential function $g$. Nonetheless, we believe the two approaches are related and could be combined in future work.

\section{Empirical Evaluation}
For the empirical evaluation of our method, we consider the setting of designing 50 sequential experiments to identify the parameters of input (design) dependent dynamical systems. The systems we examine are abstractions of robotic systems with dynamics that are described by stochastic differential equations and discretized using the Euler--Maruyama method~\citep[see, e.g.,][Section 4.3]{sarkka2023bayesian}. Hence, all scenarios involve likelihoods that are non-exchangeable conditionally Markovian densities. Moreover, these dynamical systems enforce input constraints that restrict the design space, making the sequential experimental design problem challenging.

We construct the stochastic policy $\pi_\phi$ using a mean function $m_\phi$ parameterized by a gated recurrent unit architecture~\citep[GRU,][]{cho2014learning}, with an additional learnable parameter $\Sigma_\phi$ for the variance. Then the stochastic policy $\pi_\phi$ is constructed as the law of the random variable $\xi_t = a \cdot \tanh(s_t) + b$, where $s_t \sim \mathcal{N}(m_\phi(z_{0:t}), \Sigma_\phi)$.
Here $(a,b)$ are scale and shift parameters that reflect the design constraints. Exhaustive details of the network architecture and hyperparameters are given in Section~\ref{sec:exp_details}.

We compare our algorithm, Inside-Out SMC\textsuperscript{2} (IO-SMC\textsuperscript{2}), to a few different types of design policies. We include two simple baselines - a random policy, which samples designs from a uniform distribution, and a pseudo-random binary signal (PRBS) policy, which randomly chooses between the upper and lower design limits. Additionally, we evaluate a myopic, non-amortized method, that corresponds to IO-SMC\textsuperscript{2} with a one-step look-ahead. This leads to a greedy, sub-optimal design that optimizes only for the next experiment. This approach is comparable to \citet{drovandi2013sequential}. Finally, we compare to iDAD~\citep{ivanova2021implicit} trained on the sPCE lower bound. iDAD, unlike DAD, can accommodate non-exchangeable sequential data. Although iDAD was originally developed for implicit models, we provide access to the conditional transition densities in our experiments to guarantee a fair comparison.

As evaluation metrics, we report both the sPCE bound and a nested Monte Carlo estimate of the \emph{risk-neutral} EIG in \eqref{eq:eig_factorization}. We choose to report this metric instead of the \emph{risk-sensitive} EIG from \eqref{eq:risk_sensitive_eig} in order to maintain a consistent empirical comparison with the sPCE bound, which itself does not account for risk. We compute this Monte Carlo estimate of the EIG under equally weighted sample trajectories
$\{z_{0:T}^n\}_{n=1}^N$ from the marginal $p_\phi(z_{0:T})$ as follows
\begin{equation}\label{eq:eig-estimator}
    \!\!\! \mathcal{I}(\phi) \!=\! \mathbb{E}_{p_\phi(z_{0:T})} \! \left[ \sum_{t=1}^T r_t(z_{0:t}) \right] \! \approx \! \frac{1}{N} \sum_{n=1}^N \! \sum_{t=1}^T \hat{r}_t(z_{0:t}^n). \!\!\!
\end{equation}
Here, $\hat{r}_t(z_{0:t}^n)$ is itself a particle approximation of the true stage reward in \eqref{eq:stage_reward} obtained using the filtering posterior provided by IBIS. The samples $\{z_{0:T}^n\}_{n=1}^N$ are drawn by running Algorithm~\ref{alg:i-o-smc2}, while setting $b_t^n = n$ in line~\ref{line:resampling}. This leads to a routine that generates trajectories from the marginal distribution $p_\phi(z_{0:T})$. 

For all experiments, the EIG estimate in~\eqref{eq:eig-estimator} was computed at evaluation time for $T=50$, $N = 16$ and $M = 1024$, while the sPCE bound in~\eqref{eq:spce-bound} was computed using $16$ outer samples and $L = 10^6$ regularizing samples. The statistics of these estimates are computed for $25$ seeds. Further experimental details can be found in Appendix~\ref{sec:exp_details}.

\subsection{Stochastic Pendulum}
We consider two different representations of the stochastic dynamics of a compound pendulum. The aim is to infer a vector of parameters that combines the mass and length of the pendulum by observing a sequence of states, comprised of its angular position and velocity. The design is the torque applied as input to the system.

\subsubsection{Conditionally Linear formulation}\label{sec:cond_linear_exp}
First, we consider a conditionally linear formulation of the dynamics of the compound pendulum, see Appendix~\ref{app:pendulum-linear}. In conjunction with a Gaussian prior over the parameters, this setting allows us to compare IO-SMC\textsuperscript{2} with exact posterior computation against the approximate posteriors delivered by IBIS. Details on exact posterior inference can be found in Appendix~\ref{sec:conditionally_linear}.
\begin{figure}[t]
    \centering
    \begin{tikzpicture}

\definecolor{forestgreen4416044}{RGB}{0,77,64}

\begin{axis}[
    width=8.75cm,
    height=7cm,
    grid=both,
    minor tick num=3,
    grid style={line width=.1pt, draw=gray!10},
    major grid style={line width=.1pt, draw=gray!50},
    xlabel=Number of Experiments,
    ylabel=Information Gain,
    legend style={
        nodes={scale=0.85, transform shape},
        at={(-11pt, 148.25pt)},
        anchor=north west,
    },
    legend cell align={left},
    ymin = -0.25,
    ymax = 4.0
]
    \addplot [
        black,
        thick,
        mark=o,
        mark size=3,
        error bars/.cd,
            y dir=both,y explicit,
    ] coordinates {
        (5.0, 0.07161156235090932) +- (0.0, 1.388456903851859e-15)
        (10.0, 0.5883454344117017) +- (0.0, 0.007311767064967689)
        (15.0, 0.8969540087113823) +- (0.0, 0.013946068840616034)
        (20.0, 1.116903123926479) +- (0.0, 0.028241681597221426)
        (25.0, 1.468278144575075) +- (0.0, 0.06957827043715863)
        (30.0, 1.8484101443385321) +- (0.0, 0.10779500372002392)
        (35.0, 2.1907966260130105) +- (0.0, 0.12802240485423114)
        (40.0, 2.5235308717967717) +- (0.0, 0.1308479250609831)
        (45.0, 2.9447573722608444) +- (0.0, 0.14504002884308428)
        (50.0, 3.368373462720764) +- (0.0, 0.1684394543098484)
    };
    \addlegendentry{IO-SMC\textsuperscript{2}}
    \addplot [
        black,
        thick,
        mark=square,
        mark size=3,
        error bars/.cd,
            y dir=both,y explicit,
    ] coordinates {
        (5.0, 0.13799058296021446) +- (0.0, 1.9160705273155654e-15)
        (10.0, 0.6539868433737624) +- (0.0, 0.003740368242227807)
        (15.0, 0.893829262776036) +- (0.0, 0.00888332633891218)
        (20.0, 1.1294196125936133) +- (0.0, 0.041653137869803715)
        (25.0, 1.4538769581680229) +- (0.0, 0.06836318794221811)
        (30.0, 1.8086046632861035) +- (0.0, 0.09538007556864042)
        (35.0, 2.139592639635067) +- (0.0, 0.11287474979348362)
        (40.0, 2.539831437535729) +- (0.0, 0.13137467293235863)
        (45.0, 3.0640563504500955) +- (0.0, 0.16548538189850162)
        (50.0, 3.414294158975968) +- (0.0, 0.17195802684073686)
    };
    \addlegendentry{IO-SMC\textsuperscript{2} (Exact)}
    \addplot [
        black,
        thick,
        mark=diamond,
        mark size=3,
        error bars/.cd,
            y dir=both,y explicit,
    ] coordinates {
        (5.0, 0.20272922077629982) +- (0.0, 3.8043719165540934e-15)
        (10.0, 0.7008624648284555) +- (0.0, 0.006979078789512275)
        (15.0, 0.9635595642764109) +- (0.0, 0.02359751095166349)
        (20.0, 1.1599493528757996) +- (0.0, 0.048180064392983724)
        (25.0, 1.3323088650762822) +- (0.0, 0.08424642991987262)
        (30.0, 1.4975720180824317) +- (0.0, 0.13053155871952127)
        (35.0, 1.6605851776768832) +- (0.0, 0.17956279826399424)
        (40.0, 1.8233222694557571) +- (0.0, 0.23312871435834565)
        (45.0, 1.9903116474115760) +- (0.0, 0.28970826288530127)
        (50.0, 2.1582511102958497) +- (0.0, 0.345252244354387)
    };
    \addlegendentry{PRBS}
    \addplot [
        black,
        thick,
        mark=triangle,
        mark size=4,
        error bars/.cd,
            y dir=both,y explicit,
    ] coordinates {
        (5.0, 0.07568809049562308) +- (0.0, 0.061912849605351004)
        (10.0, 0.3422215373217974) +- (0.0, 0.09120803362491624)
        (15.0, 0.5252043531375195) +- (0.0, 0.09361992978147446)
        (20.0, 0.6670799329253952) +- (0.0, 0.09414091611192746)
        (25.0, 0.7919381957975652) +- (0.0, 0.0969310735572835)
        (30.0, 0.904940907432193) +- (0.0, 0.10586752727472198)
        (35.0, 1.01201445880981) +- (0.0, 0.12050755169158237)
        (40.0, 1.1166571141420853) +- (0.0, 0.14288836185741793)
        (45.0, 1.224183430338142) +- (0.0, 0.17347366564932146)
        (50.0, 1.3332246944262631) +- (0.0, 0.2077544154857676)
    };
    \addlegendentry{Random}
\end{axis}
\end{tikzpicture} 
    \vspace{-0.25cm}
    \caption{Accumulation of the information gain computed in closed form for different policies on the conditionally linear stochastic pendulum with a Gaussian prior. We report the mean and standard deviation over $512$ realizations.}
    \label{fig:linear_pendulum_info_gain}
    \vspace{0.25cm}
\end{figure}
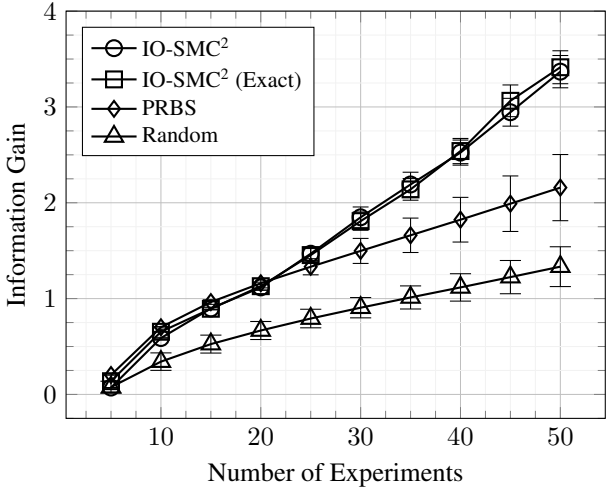

In this conjugate setting, the information gain can be computed in closed form, given observations from sequential experiments. Figure~\ref{fig:linear_pendulum_info_gain} depicts the mean and standard deviation of the IG over experiments for different policies, obtained by simulating $512$ sets of sequential experiments for different samples from the parameter prior. The amortized policies are superior to the random and PRBS policies. These results corroborate Table~\ref{tab:linear-pendulum}, which reports EIG estimates and sPCE bounds for all considered policies. The two variants of our algorithm outperform all considered baselines on both metrics.
\begin{figure}[t]
    \centering
    \begin{tikzpicture}

\definecolor{forestgreen4416044}{RGB}{0,77,64}
\definecolor{steelblue31119180}{RGB}{30,136,229}

\begin{axis}[
    width=8.75cm,
    height=7cm,
    grid=both,
    minor tick num=3,
    grid style={line width=.1pt, draw=gray!10},
    major grid style={line width=.1pt, draw=gray!50},
    xlabel=Training Epochs,
    ylabel=EIG Estimate,
    legend style={
        nodes={scale=0.85, transform shape},
        at={(181pt, 6pt)},
        anchor=south east
    },
    legend cell align={left},
    ymin = -0.25,
    ymax = 4.0
]
    \addplot [
        black,
        thick,
        mark=*,
        mark size=2,
        error bars/.cd,
            y dir=both,y explicit,
    ] coordinates {
        (1.0, 0.034333807542068884) +- (0.0, 0.012777299220350273)
        (2.0, 0.06190272951406024) +- (0.0, 0.041125862010712166)
        (3.0, 0.11727582612189716) +- (0.0, 0.10096020848494787)
        (4.0, 0.2004209286341635) +- (0.0, 0.18162252906435267)
        (5.0, 0.3406186333940914) +- (0.0, 0.2997401752679322)
        (6.0, 0.5573050922683379) +- (0.0, 0.3840611663878479)
        (7.0, 0.8897275094548454) +- (0.0, 0.5129228753066545)
        (8.0, 1.2443763895295819) +- (0.0, 0.5081138484891182)
        (9.0, 1.606525130059341) +- (0.0, 0.5104077395126563)
        (10.0, 1.9173756475175099) +- (0.0, 0.4836592704155913)
        (11.0, 2.211450535302662) +- (0.0, 0.4323442368388904)
        (12.0, 2.435317376017777) +- (0.0, 0.4021791690206066)
        (13.0, 2.5998102614784413) +- (0.0, 0.3947074521685271)
        (14.0, 2.7760400045622577) +- (0.0, 0.3420158830347885)
        (15.0, 2.9170793696036346) +- (0.0, 0.29959060391815234)
        (16.0, 3.0077438084818966) +- (0.0, 0.2364101510573952)
        (17.0, 3.0690922792702513) +- (0.0, 0.2147034842241701)
        (18.0, 3.128821145404627) +- (0.0, 0.195293753814367)
        (19.0, 3.2073955167754917) +- (0.0, 0.16932769658578517)
        (20.0, 3.2392485707811933) +- (0.0, 0.16567695629657594)
        (21.0, 3.281875172856365) +- (0.0, 0.11760737070750146)
        (22.0, 3.3124726735518846) +- (0.0, 0.11229711333583155)
        (23.0, 3.344093873048434) +- (0.0, 0.13120189255115217)
        (24.0, 3.35134303249425) +- (0.0, 0.14356233652459313)
        (25.0, 3.3794267957236848) +- (0.0, 0.1190090522549283)
    };
    \addlegendentry{IO-SMC\textsuperscript{2}}
    \addplot [
        darkgray,
        thick,
        mark=square,
        mark size=2,
        mark options={solid},
        error bars/.cd,
            y dir=both,y explicit,
    ] coordinates {
        (1.0, 0.035616459921870146) +- (0.0, 0.011175661475871658)
        (2.0, 0.06255340543491762) +- (0.0, 0.03814105597338317)
        (3.0, 0.1309518027091984) +- (0.0, 0.12157131762788337)
        (4.0, 0.2586871309300582) +- (0.0, 0.2151222833944215)
        (5.0, 0.44646238088696777) +- (0.0, 0.33662175899913205)
        (6.0, 0.6776734496709467) +- (0.0, 0.39861236826957847)
        (7.0, 1.0187048559544718) +- (0.0, 0.47049268357706203)
        (8.0, 1.4322134854109552) +- (0.0, 0.5348965845543251)
        (9.0, 1.731743540543758) +- (0.0, 0.5611761263257665)
        (10.0, 2.0406020300241994) +- (0.0, 0.5667667842702762)
        (11.0, 2.289243545501996) +- (0.0, 0.5106520877981163)
        (12.0, 2.5335690597157177) +- (0.0, 0.469902868592756)
        (13.0, 2.71159862720185) +- (0.0, 0.3973451564294755)
        (14.0, 2.8778384178961938) +- (0.0, 0.3395995055655779)
        (15.0, 3.010930990524492) +- (0.0, 0.2833562263613981)
        (16.0, 3.1075931178991874) +- (0.0, 0.25257263549620945)
        (17.0, 3.2032788285783376) +- (0.0, 0.22044344979835115)
        (18.0, 3.261250930115647) +- (0.0, 0.20119460761380434)
        (19.0, 3.313204912972115) +- (0.0, 0.1778632299551653)
        (20.0, 3.364265376050372) +- (0.0, 0.15424874205724265)
        (21.0, 3.401897945221866) +- (0.0, 0.15319896726811416)
        (22.0, 3.412911801398632) +- (0.0, 0.14389392519800442)
        (23.0, 3.4383464868380598) +- (0.0, 0.14396811488348807)
        (24.0, 3.493491500043553) +- (0.0, 0.12800523819607057)
        (25.0, 3.504203663367869) +- (0.0, 0.13639683101777458)
    };
    \addlegendentry{IO-SMC\textsuperscript{2} (Exact)}
\end{axis}
\end{tikzpicture} 
    \vspace{-0.25cm}
    \caption{Training progression of the IO-SMC\textsuperscript{2} policy and its exact variant on the conditionally linear stochastic pendulum. At every epoch, we evaluate the EIG estimate using the mean policy. We report the mean and standard deviation over $25$ seeds.}
    \label{fig:linear_pendulum_eig_training}
    \vspace{-0.25cm}
\end{figure}
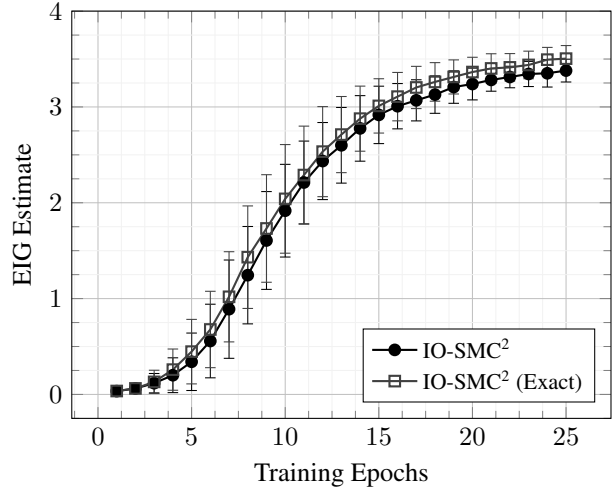
\begin{table}[t]
    \vspace{-0.15cm}
    \caption{EIG estimates and sPCE lower bounds on the conditionally linear pendulum experiment for the considered methods. We report the mean $\pm$ standard deviation over $25$ seeds.}
    \label{tab:linear-pendulum}
    \vskip 0.15in
    \centering
    \begin{tabular}{lcc}
        \toprule
        Policy & EIG Estimate \eqref{eq:eig-estimator} & sPCE \\
        \midrule
        Random & $1.37 \pm 0.08$ & $1.44 \pm 0.35$ \\
        Myopic & $1.45 \pm 0.12$ & $1.41 \pm 0.32$ \\
        PRBS & $2.24 \pm 0.19$ & $2.33 \pm 0.32$ \\ 
        iDAD & $2.58 \pm 0.17$ & $2.53 \pm 0.35$ \\
        IO-SMC\textsuperscript{2} & $3.53 \pm 0.15$ &  $3.66 \pm 0.44$ \\
        IO-SMC\textsuperscript{2} (Exact) & $3.63 \pm 0.18$ & $3.64 \pm 0.41$ \\
        \bottomrule
    \end{tabular}
    \vskip 0.1in
\end{table}

To evaluate the stability of our algorithm during training, we trained $25$ different policy networks using IO-SMC\textsuperscript{2} and its exact variant. The means and standard deviations of EIG estimates obtained after each training epoch are depicted in Figure~\ref{fig:linear_pendulum_eig_training}. Notice that the standard deviation shrinks over training epochs, reaching a final value comparable to that observed in EIG estimates for a single policy~(Table~\ref{tab:linear-pendulum}), implying consistency across training runs.

Figure~\ref{fig:linear_pendulum_info_gain} and Figure~\ref{fig:linear_pendulum_eig_training} demonstrate that the two versions of IO-SMC\textsuperscript{2} deliver comparable performance during learning and inference, empirically validating the use of IBIS in general non-conjugate settings. 

\subsubsection{Nonlinear formulation}
We now consider the standard, nonlinear version of the stochastic pendulum. We choose a log-normal prior for the parameters, ensuring that the mass and length of the pendulum can only take positive values, see Appendix~\ref{app:pendulum-non-linear}. As a result, the exact version of our algorithm is no longer applicable, and we have to use IO-SMC\textsuperscript{2} in its general form.

EIG estimates and sPCE lower bounds for different policies in this environment are reported in Table~\ref{tab:nonlinear_pendulum}. The policy trained using IO-SMC\textsuperscript{2} outperforms all baseline methods. An example of a trajectory generated by the mean of the trained policy is given in Figure~\ref{fig:nonlinear_pendulum_sample_trajectory}. The policy has learned to swing the pendulum to achieve greater and greater angular velocities by alternating between maximum positive and negative designs, thus exploring more of the phase space.

\subsection{Stochastic Cart-Pole}
We now consider a cart-pole system with additive noise in the acceleration of the cart. The unknown parameters are the masses of the cart and pole and the length of the pole. The outcome of every experiment is an observation of the cart-pole state consisting of the position and velocity of the cart and the angular position and angular velocity of the pole. The design is the force applied to the cart at discrete intervals. The prior on the parameters is again a log-normal distribution. Further details on the experimental setup are given in Appendix~\ref{app:cart-pole}.

\begin{figure}[t]
    \centering
    \begin{tikzpicture}
\begin{axis}[
    width=8.75cm,
    height=6.1cm,
    grid=both,
    minor tick num=3,
    grid style={line width=.1pt, draw=gray!10},
    major grid style={line width=.1pt, draw=gray!50},
    xlabel=Number of Experiments,
    legend style={
        nodes={scale=0.85, transform shape},
        at={(-11pt, 122pt)},
        anchor=north west
    },
    table/col sep=comma,
    ymin=-3.0,
    ymax=5.0
]
    \pgfplotstableread{figures/nonlinear_pendulum_csmc_ibis_trajectory.csv}\loadedtable
    \addplot [black, thick, densely dotted] table [x=t,y=q] {\loadedtable};
    \addlegendentry{$q$}
    \addplot [black, thick, dashed] table [x=t,y=q_dot] {\loadedtable};
    \addlegendentry{$\dot{q}$}
    \addplot [black, thick] table [x=t,y=u] {\loadedtable};
    \addlegendentry{$\xi$}
\end{axis}
\end{tikzpicture}
    \vspace{-0.25cm}
    \caption{A sample experiment trajectory generated by the amortized policy during deployment on the nonlinear stochastic pendulum environment. $q$ is the angle of the pendulum from the vertical, $\dot{q}$ is the angular velocity and $\xi$ is the design.}
    \label{fig:nonlinear_pendulum_sample_trajectory}
    \vspace{-0.20cm}
\end{figure}
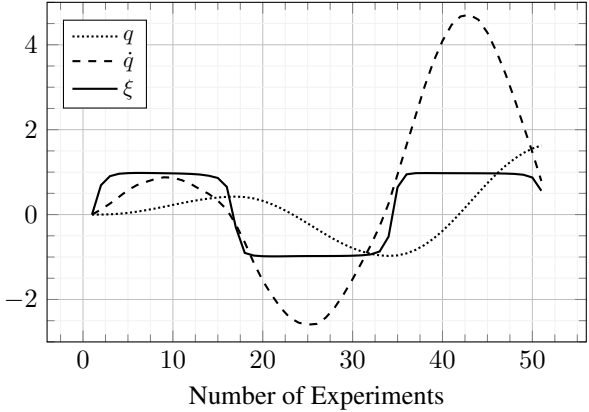
\begin{table}[t]
    \vspace{-0.20cm}
    \caption{EIG estimates and sPCE lower bounds on the nonlinear pendulum experiment for the considered methods. We report the mean $\pm$ standard deviation over $25$ seeds.}
    \label{tab:nonlinear_pendulum}
    \vskip 0.1in
    \centering
    \begin{tabular}{lcc}
        \toprule
        Policy & EIG Estimate \eqref{eq:eig-estimator} & sPCE \\
        \midrule
        Random & $2.12 \pm 0.21$ & $2.28 \pm 0.25$ \\
        Myopic & $2.15 \pm 0.18$ & $2.27 \pm 0.32$ \\
        PRBS & $3.00 \pm 0.20$ & $2.94 \pm 0.33$ \\ 
        iDAD & $3.01 \pm 0.29$ & $3.18 \pm 0.41$ \\
        IO-SMC\textsuperscript{2} & $3.72 \pm 0.17$ & $3.77 \pm 0.38$ \\
        \bottomrule
    \end{tabular}
    \vskip 0.1in
\end{table}

Table~\ref{tab:cartpole} reports the performance of each policy under consideration. The policy trained using IO-SMC\textsuperscript{2} achieves the highest mean EIG estimate at $21.23$. This experiment demonstrates the primary drawback of the sPCE bound; one needs approximately $1.3$ billion regularizing samples to yield an sPCE bound of $21$. To accommodate this number of samples, the implementation of \citet{ivanova2021implicit}, which relies on parallel computation, would require about $\SI{660}{\giga\byte}$ of memory with double-precision floats. In contrast, IO-SMC\textsuperscript{2} requires just $1024$ inner samples to obtain our estimate in Table~\ref{tab:cartpole}. Thus, for sequential experimental design problems with high EIG values, the advantage of IO-SMC\textsuperscript{2} is clear. The sample efficiency of our EIG estimator compared to sPCE is further demonstrated in Table~\ref{tab:eig-spce-comparison} in Appendix~\ref{app:cart-pole}.
Nevertheless, we note that the sPCE lower bound is still valuable for training despite its bias. Indeed, although iDAD underperforms compared to IO-SMC\textsuperscript{2}, it achieves a higher EIG at evaluation time than the expected upper bound of approximately $10$, corresponding to the $2^{14}$ regularizing samples used at training time.

A sample trajectory generated by the trained policy mean is depicted in Figure~\ref{fig:sample-cartpole-trajectory}. As in the case of the pendulum, the policy has learned to alternate between the design limits to explore the phase space efficiently.
\begin{figure}[t]
    \centering
    \begin{tikzpicture}
\begin{axis}[
    width=8.75cm,
    height=6.1cm,
    grid=both,
    minor tick num=3,
    grid style={line width=.1pt, draw=gray!10},
    major grid style={line width=.1pt, draw=gray!50},
    xlabel=Number of Experiments,
    legend style={
        nodes={scale=0.85, transform shape},
        at={(-11pt, 112pt)},
        anchor=north west
    },
    table/col sep=comma,
]
    \pgfplotstableread{figures/cartpole_ibis_csmc_trajectory.csv}\loadedtable
    \addplot [darkgray, thick, loosely dotted] table [x=t,y=s] {\loadedtable};
    \addlegendentry{$s$}
    \addplot [black, thick, dashed] table [x=t,y=q] {\loadedtable};
    \addlegendentry{$q$}
    \addplot [darkgray, thick, densely dotted] table [x=t,y=ds] {\loadedtable};
    \addlegendentry{$\dot{s}$}
    \addplot [black, thick, densely dashdotted] table [x=t,y=dq] {\loadedtable};
    \addlegendentry{$\dot{q}$}
    \addplot [black, thick] table [x=t,y=u] {\loadedtable};
    \addlegendentry{$\xi$}
\end{axis}
\end{tikzpicture} 
    \vspace{-0.25cm}
    \caption{A sample experiment trajectory generated by the policy during deployment on the stochastic cart-pole environment. Here, $s$ and $\dot{s}$ are the position and velocity of the cart respectively, $q$ is the angle of the pole, $\dot{q}$ is its angular velocity and $\xi$ is the design.}
    \label{fig:sample-cartpole-trajectory}
    \vspace{-0.2cm}
\end{figure}
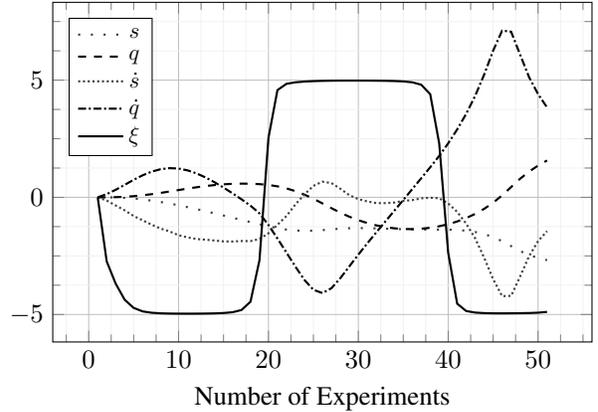
\begin{table}[t]
    \vspace{-0.20cm}
    \caption{EIG estimates and sPCE lower bounds on the stochastic cart-pole experiment for the considered methods. We report the mean $\pm$ standard deviation over $25$ seeds. sPCE bounds for all policies hit the upper bound of $\log(10^6) \approx 13.82$, and precise estimates would need at least $\exp(21) \approx 1.3$ billion regularizing samples, far beyond our hardware limits.}
    \label{tab:cartpole}
    \vskip 0.1in
    \centering
    \begin{tabular}{lcc}
        \toprule
        Policy & EIG Estimate \eqref{eq:eig-estimator} & sPCE \\
        \midrule
        Random & $16.81 \pm 0.77$ & $13.72 \pm 0.08$ \\
        Myopic & $16.53 \pm 0.71$ & $13.74 \pm 0.10$ \\
        PRBS & $18.28 \pm 0.50$ & $13.80 \pm 0.03$ \\
        iDAD & $18.99 \pm 0.68$ & $13.81 \pm 0.01$ \\
        IO-SMC\textsuperscript{2} & $21.23 \pm 0.62$ & $13.82 \pm 0.00$ \\
        \bottomrule
    \end{tabular}
    \vskip 0.1in
\end{table}

\subsection{Stochastic Double-Link}
Our final experiment uses a stochastic double-link~(double pendulum) system with the four unknowns being the masses and lengths of both links. The design is two-dimensional, corresponding to the torque applied at each of the two actuated joints. The dynamical equations and policy hyperparameters are given in Appendix~\ref{app:double-pendulum}.
\begin{figure}[t]
    \centering
    \begin{tikzpicture}
\begin{axis}[
    width=8.75cm,
    height=6.1cm,
    grid=both,
    minor tick num=3,
    grid style={line width=.1pt, draw=gray!10},
    major grid style={line width=.1pt, draw=gray!50},
    xlabel=Number of Experiments,
    legend columns=2, 
    legend style={
        nodes={scale=0.85, transform shape},
        at={(-11pt, 112pt)},
        anchor=north west
    },
    table/col sep=comma,
]
    \pgfplotstableread{figures/double_pendulum_ibis_csmc_trajectory.csv}\loadedtable
    \addplot [gray, thick, loosely dotted] table [x=t,y=q1] {\loadedtable};
    \addlegendentry{$q_1 \,\,$}
    \addplot [gray, thick, dashed] table [x=t,y=q2] {\loadedtable};
    \addlegendentry{$q_2$}
    \addplot [gray, thick, densely dotted] table [x=t,y=dq1] {\loadedtable};
    \addlegendentry{$\dot{q}_1 \,\,$}
    \addplot [gray, thick, densely dashed] table [x=t,y=dq2] {\loadedtable};
    \addlegendentry{$\dot{q}_2$}
    \addplot [black, thick, densely dashdotted] table [x=t,y=u1] {\loadedtable};
    \addlegendentry{$\xi_1 \,\,$}
    \addplot [black, thick] table [x=t,y=u2] {\loadedtable};
    \addlegendentry{$\xi_2$}
\end{axis}
\end{tikzpicture} 
    \vspace{-0.25cm}
    \caption{A sample experiment trajectory generated by the policy during deployment on the stochastic double-link environment. $q_1$ and $q_2$ are the angles from the vertical for the two links, $\dot{q}_1$ and $\dot{q}_2$ their respective angular velocities, and $\xi_1$ and $\xi_2$ are the designs.}
    \label{fig:double-pendulum-trajectory}
    \vspace{-0.35cm}
\end{figure}
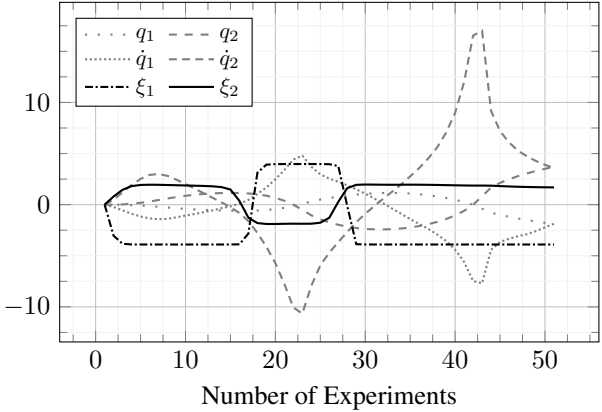
Figure~\ref{fig:double-pendulum-trajectory} plots a sample trajectory generated by a policy trained using our algorithm. The optimal policy implements two coordinated switches in the design dimensions to bring the system into informative states. In this experiment, iDAD also learns a similar policy, as reflected in the EIG estimates in Table~\ref{tab:double-link}.
\begin{table}[t]
    \caption{EIG estimates and sPCE lower bounds on the stochastic double-link experiment for the considered methods. We report the mean $\pm$ standard deviation over $25$ seeds.}
    \label{tab:double-link}
    \vskip 0.1in
    \centering
    \begin{tabular}{lcc}
        \toprule
        Policy & EIG Estimate \eqref{eq:eig-estimator} & sPCE \\
        \midrule
        Random & $7.81 \pm 0.40$ & $7.79 \pm 0.43$ \\
        Myopic & $8.00 \pm 0.44$ & $8.13 \pm 0.50$ \\
        PRBS & $5.25 \pm 0.26$ & $5.12 \pm 0.42$ \\
        iDAD & $11.73 \pm 0.45$ & $11.52 \pm 0.36$ \\
        IO-SMC\textsuperscript{2} & $11.53 \pm 0.49$ & $11.45 \pm 0.42$ \\
        \bottomrule
    \end{tabular}
    \vskip 0.1in
\end{table}

\section{Discussion and Limitations}
We have introduced a novel method of amortized sequential Bayesian experimental design, taking inspiration from the control-as-inference framework. We cast the optimization of sequential designs as a smoothing problem in a non-Markovian state-space model. To perform inference in this model, we developed a novel nested particle filtering algorithm, which we call Inside-Out SMC\textsuperscript{2}. Our approach naturally lends itself to amortization via likelihood optimization in the form of Markovian score climbing. While we have used Inside-Out SMC\textsuperscript{2} in the context of sequential BED, we believe it may find uses in other settings, where one wishes to obtain pathwise smoothing trajectories under parameter-marginalized models. 

Our experimental evaluation shows that our approach holds promise as a generic and efficient way to learn experimental design policies: it is amortized, non-myopic, widely applicable, and easy to train. Our learned policies outperform the main alternative~\citep[iDAD,][]{ivanova2021implicit} while only requiring a fraction of the number of samples at both training and evaluation time. In particular, the sPCE bound, used therein, requires the number of samples to grow exponentially with the maximal EIG value. This can make it unsuitable as a learning objective in certain dynamical systems. On the contrary, IO-SMC\textsuperscript{2} can compute approximations to the EIG, no matter its value, with a relatively small number of particles, making it a better-behaved and more viable alternative.

One limitation of our approach is the requirement to evaluate the conditional transition densities in closed form. Thus, our method is unsuitable for sequential experimental design problems in dynamical models with intractable densities; for instance, choosing optimal measurement times in a compartmental epidemic model modeled as a Markov jump process~\citep{whitehouse2023consistent}.

\emph{Time complexity.} IO-SMC\textsuperscript{2} has a time complexity of $\mathcal{O}(NMT^2)$. CSMC requires $N \propto T$ samples to be stable~\citep[Proposition 5]{lindsten2015uniform} for an increasing number of time steps, making our algorithm $\mathcal{O}(MT^3)$ when accounting for statistical stability. The outer SMC loop can be trivially parallelized, reducing it to $\mathcal{O}(MT^2)$.

\emph{Scalability.} SMC samplers, which IBIS is an instance of, are known to scale well with the dimension~\citep{dai2022invitation}. Since the random walk Metropolis kernel, which we use within IBIS, scales reasonably well with the dimensionality of the problem~\citep{gelman1997weak}, we expect our algorithm to scale well in the parameter dimension. However, scaling in the number of state dimensions might be more problematic, as we are using a bootstrap proposal, known to degenerate when the dimension of the observations is large. This is related to the informativeness of the potential function, which in our case can be regulated by controlling the tempering parameter $\eta$.

\emph{Tempering.} The choice of the tempering parameter $\eta$ remains an open research question, that touches on the interaction between the optimism in the policy amortization step and the variance of the particle filtering weights within IO-SMC\textsuperscript{2}. Addressing this issue requires solving a bias-variance trade-off that is inherent to risk-sensitive objectives. In principle, the most favorable outcome would be to obtain an inference problem that can scale to any value of $\eta$.

\section*{Individual Contributions}
The original idea for the article was conceived by HA and developed jointly with AC and SI. SI developed and implemented Inside-Out SMC\textsuperscript{2} with AC's guidance, and the proof of its consistency is due to AC. The experiments are due equally to HA and SI. HA supervised the project. SI wrote the initial manuscript, and all authors contributed to revisions. SS reviewed and validated the technical details.

\newpage
\section*{Acknowledgements}
SI gratefully acknowledges funding from the Research Council of Finland. HA acknowledges funding by the Finnish Center for Artificial Intelligence (FCAI).

\section*{Impact Statement}
This paper presents work whose goal is to advance the field of Machine Learning. There are many potential societal consequences of our work, none of which we feel must be specifically highlighted here.

\bibliography{references}
\bibliographystyle{icml2024}

\newpage
\appendix
\onecolumn
\section{Proof of Proposition~\ref{prop:eig_factorization}}
\label{app:prop_1_proof}
\begin{proof}
We start with the definition of the expected information gain in the terminal reward framework,
\begin{align*}
    \mathcal{I}(\phi) &\coloneq \mathbb{E}_{p_{\phi}(z_{0:T})} \Bigl[ \mathbb{H}\bigl[p(\theta)\bigr] - \mathbb{H}\bigl[p(\theta \mid z_{0:T})\bigr] \Bigr] \\
    &= \mathbb{E}_{p_{\phi}(z_{0:T}, \theta)} \left[\log \frac{p(\theta \mid z_{0:T})}{p(\theta)}\right].
\end{align*}
With repeated applications of Bayes' rule, we can write this in an equivalent form as
\begin{equation}\label{eq:eig_alternative_definition}
    \mathcal{I}(\phi) = \mathbb{E}_{p_{\phi}(z_{0:T}, \theta)}\bigl[\log p_{\phi}(z_{0:T} \mid \theta)\bigr] - \mathbb{E}_{p_{\phi}(z_{0:T})} \bigl[\log p_{\phi}(z_{0:T})\bigr].
\end{equation}
Let us look at the first term in~\eqref{eq:eig_alternative_definition}. From~\eqref{eq:joint_density}, we know that the conditional trajectory likelihood is
\begin{equation*}
  p_{\phi}(z_{0:T} \mid \theta) = p(x_0) \prod_{t=1}^T f(x_t\mid x_{t-1}, \xi_{t-1}, \theta)\,\pi_\phi(\xi_{t-1} \mid z_{0:t-1}).
\end{equation*}
We can then evaluate
\begin{align}
    T_1 &\coloneq \mathbb{E}_{p_{\phi}(z_{0:T}, \theta)}\bigl[\log p_{\phi}(z_{0:T} \mid \theta)\bigr] \nonumber\\
    &= \mathbb{E}_{p_{\phi}(z_{0:T}, \theta)} \left[ \log p(x_0) + \sum_{t=1}^T \log f(x_t\mid x_{t-1}, \xi_{t-1}, \theta) \right] + \mathbb{E}_{p_{\phi}(z_{0:T})} \left[ \sum_{t=1}^T \log \pi_\phi(\xi_{t-1} \mid z_{0:t-1}) \right] \nonumber \\
    &= -\mathbb{H}[p(x_0)] + \mathbb{E}_{p_{\phi}(z_{0:T}, \theta)} \left[ \sum_{t=1}^T \log f(x_t\mid x_{t-1}, \xi_{t-1}, \theta) \right] + \mathbb{E}_{p_{\phi}(z_{0:T})} \left[ \sum_{t=1}^T \log \pi_\phi(\xi_{t-1} \mid z_{0:t-1}) \right]. \label{eq:eig_term_1}
\end{align}
For the second term in the above equation, we get
\begin{align*}
        \mathbb{E}_{p_{\phi}(z_{0:T}, \theta)} \left[ \sum_{t=1}^T \log f(x_t\mid x_{t-1}, \xi_{t-1}, \theta) \right]
        &= \sum_{t=1}^T \mathbb{E}_{p_{\phi}(z_{0:t}, \theta)} \log f(x_t\mid x_{t-1}, \xi_{t-1}, \theta) \\
        &= \sum_{t=1}^T \mathbb{E}_{p_{\phi}(z_{0:t})\,p(\theta \mid z_{0:t})} \log f(x_t\mid x_{t-1}, \xi_{t-1}, \theta) \\
        &= \mathbb{E}_{p_{\phi}(z_{0:T})} \left[ \sum_{t=1}^T \mathbb{E}_{p(\theta \mid z_{0:t})} \log f(x_t\mid x_{t-1}, \xi_{t-1}, \theta) \right] \\
        &\eqqcolon \mathbb{E}_{p_{\phi}(z_{0:T})} \left[ \sum_{t=1}^T \alpha_t(z_{0:t}) \right],
\end{align*}
where we have made use of the fact that the conditional dynamics is Markovian in multiple places. We are then left with
\begin{equation}
    T_1 = -\mathbb{H}[p(x_0)] + \mathbb{E}_{p_{\phi}(z_{0:T})} \left[ \sum_{t=1}^T \alpha_t(z_{0:t}) \right] + \mathbb{E}_{p_{\phi}(z_{0:T})} \left[ \sum_{t=1}^T \log \pi_\phi(\xi_{t-1} \mid z_{0:t-1}) \right].
\end{equation}

Let's now look at the second term in~\eqref{eq:eig_alternative_definition}. For that, we first need the following factorization of the marginal trajectory distribution
\begin{align*}
    \log p_{\phi}(z_{0:T}) &= \log \left[p(z_0) \prod_{t=1}^T p_{\phi}(z_t \mid z_{0:t-1})\right] \\
    &= \log \left[p(x_0) \prod_{t=1}^T p_{\phi}(x_t, \xi_{t-1} \mid z_{0:t-1})\right] \\
    &= \log p(x_0) + \sum_{t=1}^T \Bigl[\log p(x_t \mid z_{0:t-1}, \xi_{t-1}) + \log \pi_\phi(\xi_{t-1} \mid z_{0:t-1})\Bigr],
\end{align*}
where
\begin{align*}
        p(x_t \mid  z_{0:t-1}, \xi_{t-1}) &= \int_\Theta p(x_t, \theta \mid  z_{0:t-1}, \xi_{t-1})\,\dd\theta \\
        &= \int_\Theta p(x_t \mid z_{0:t-1}, \xi_{t-1}, \theta)\,p(\theta \mid  z_{0:t-1}, \xi_{t-1})\,\dd\theta \\
        &= \int_\Theta f(x_t \mid x_{t-1}, \xi_{t-1}, \theta)\,p(\theta \mid  z_{0:t-1})\,\dd \theta.
\end{align*}
In the last line, we have used the fact that $\theta$ is conditionally independent of $\xi_{t-1}$ given $ z_{0:t-1}$ (since the policy is independent of $\theta$). We can now compute the second term of the EIG as
\begin{align}
    T_2 &\coloneq \mathbb{E}_{p_{\phi}(z_{0:T})} \bigl[\log p_{\phi}(z_{0:T}) \bigr] \nonumber \\
    &= \mathbb{E}_{p_{\phi}(z_{0:T})} \left[\log p(x_0) + \sum_{t=1}^T \Bigl[\log p(x_t \mid  z_{0:t-1}, \xi_{t-1}) + \log \pi_\phi(\xi_{t-1} \mid  z_{0:t-1})\Bigr]\right] \nonumber \\
    &= -\mathbb{H}[p(x_0)] + \mathbb{E}_{p_{\phi}(z_{0:T})} \left[ \sum_{t=1}^T \log p(x_t \mid  z_{0:t-1}, \xi_{t-1}) \right] + \mathbb{E}_{p_{\phi}(z_{0:T})} \left[ \sum_{t=1}^T \log \pi_\phi(\xi_{t-1} \mid z_{0:t-1}) \right].
\end{align}

The full expression for $\mathcal{I}(\phi)$ is hence
\begin{align}
    \mathcal{I}(\phi) &= T_1 - T_2 \nonumber \\
    &= \mathbb{E}_{p_{\phi}(z_{0:T})} \left[ \sum_{t=1}^T \alpha_t(z_{0:t}) \right]
    - \mathbb{E}_{p_{\phi}(z_{0:T})} \left[ \sum_{t=1}^T \log p(x_t \mid  z_{0:t-1}, \xi_{t-1}) \right] \nonumber \\
    &= \mathbb{E}_{p_{\phi}(z_{0:T})} \left[ \sum_{t=1}^T \Bigl\{ \alpha_t(z_{0:t}) + \beta_{t}(z_{0:t}) \Bigr\} \right] \\
    &= \mathbb{E}_{p_{\phi}(z_{0:T})} \left[ \sum_{t=1}^T r_t(z_{0:t}) \right],
\end{align}
where we have defined $\beta_{t}(z_{0:t}) = - \log p(x_t \mid  z_{0:t-1}, \xi_{t-1})$ and $r_t(z_{0:t}) = \alpha_t(z_{0:t}) + \beta_{t}(z_{0:t})$. This concludes the proof for the first part of Proposition~\ref{prop:eig_factorization}.

Let us now assume that our model has additive, constant noise in the dynamics (noise that is independent of the state, design and $\theta$ parameters). Under this assumption, the entropy $\mathbb{H}[f(x_t\mid x_{t-1}, \xi_{t-1}, \theta)]$ is a constant.
Let us now go back to the second term in~\eqref{eq:eig_term_1},
\begin{align*}
    \mathbb{E}_{p_{\phi}(z_{0:T}, \theta)} \left[ \sum_{t=1}^T \log f(x_t\mid x_{t-1}, \xi_{t-1}, \theta) \right] &= \sum_{t=1}^T \mathbb{E}_{p(\theta) \, p_{\phi}(z_{0:T} \mid \theta)} \left[ \log f(x_t\mid x_{t-1}, \xi_{t-1}, \theta) \right] \\
    & \hspace{-2cm} \begin{aligned}
        &= \sum_{t=1}^T \mathbb{E}_{p(\theta) \, p_{\phi}(x_{0:t}, \xi_{0:t-1} \mid \theta)} \left[ \log f(x_t\mid x_{t-1}, \xi_{t-1}, \theta) \right] \\
        &= \sum_{t=1}^T \mathbb{E}_{p(\theta) \, p_{\phi}(x_{0:t-1}, \xi_{0:t-1} \mid \theta)} \left[ \int_\mathcal{X} \log f(x_t \mid x_{t-1}, \xi_{t-1}, \theta) \, f(x_t \mid x_{t-1}, \xi_{t-1}, \theta) \, \dd x_t \right] \\
        &= \sum_{t=1}^T \mathbb{E}_{p(\theta) \, p_{\phi}(x_{0:t-1}, \xi_{0:t-1} \mid \theta)} \bigl[-\mathbb{H}[f(x_t\mid x_{t-1}, \xi_{t-1}, \theta)]\bigr] \\
        &= \textrm{constant}.
    \end{aligned}
\end{align*}
Using '$\equiv$' to denote equality up to an additive constant, we now have
\begin{equation*}
    T_1 \equiv -\mathbb{H}[p(x_0)] + \mathbb{E}_{p_{\phi}(z_{0:T})} \left[ \sum_{t=1}^T \log \pi_\phi(\xi_{t-1} \mid z_{0:t-1}) \right],
\end{equation*}
and hence,
\begin{align}
    \mathcal{I}(\phi) &= T_1 - T_2 \nonumber \\
    &\equiv \mathbb{E}_{p_{\phi}(z_{0:T})} \left[ \sum_{t=1}^T - \log p(x_t \mid  z_{0:t-1}, \xi_{t-1}) \right] \\
    &= \mathbb{E}_{p_{\phi}(z_{0:T})} \left[ \sum_{t=1}^T \beta_{t}(z_{0:t}) \right]
\end{align}
as required.
\end{proof}

\section{Proof of Proposition~\ref{prop:consistency}}\label{app:consistency}
    We prove this result by induction over $t$. We note that this result is not directly implied by the existing classical sequential Monte Carlo theory due to the dependency of the likelihood term on the filtering distribution over $\theta$. 

    We will make use of the following assumptions.
    \begin{assumption}[]\label{ass:bounded-func}
        For all $z_{0:t+1}, t\geq 1$ there exists $\alpha > 0$ such that, for all $\theta$, $0 < f(x_{t+1} \mid x_t, \xi_t, \theta) < \alpha$.
    \end{assumption}
    \begin{assumption}[]\label{ass:strong-mixing}
        For all $z_{0:t}, t\geq 1$ there exists an integrable function $\beta_{t+1}(z_{t+1})$ such that, for all $\theta$, $ 0 < p(z_{t+1} \mid z_{0:t}, \theta) < \beta_{t+1}(z_{t+1})$.
    \end{assumption}
    \begin{proof}
        For simplicity, we assume that resampling happens at each step of IBIS in Algorithm~\ref{alg:ibis}.
        The result is clear for $t=0$ under the law of large numbers. Now assume it is true for a given $t$, then, applying Proposition~\ref{prop:consistency} to the test function $\theta \mapsto f(x_{t+1} \mid x_t, \xi_t, \theta)$ we have, thanks to Assumption~\ref{ass:bounded-func}
        \begin{equation}
            \frac{1}{M} \sum_{m=1}^M f(x_{t+1} \mid x_t, \xi_t, \theta_{t}^m) \to \mathbb{E}_{p(\theta_t \mid z_{0:t})}\left[f(x_{t+1} \mid x_t, \xi_t, \theta_{t})\right] = p(x_{t+1} \mid z_{0:t}, \xi_t)
        \end{equation}
        so that, with probability $1$,
        \begin{equation}
            \exp\left\{-\eta \log \frac{1}{M} \sum_{m=1}^M f(x_{t+1} \mid x_t, \xi_t, \theta_{t}^m)\right\} \to \exp\left\{-\eta \log p(x_{t+1} \mid z_{0:t}, \xi_t)\right\}.
        \end{equation}
        Applying the induction hypothesis to $\theta \mapsto p(z_{t+1} \mid z_{0:t}, \theta)$, we have with probability $1$,
        \begin{equation}
            \frac{1}{M} \sum_{m=1}^M p(z_{t+1} \mid z_{0:t}, \theta_{t}^m) \to p(z_{t+1} \mid z_{0:t}).
        \end{equation}
        As a consequence, 
        \begin{equation}
        \left[
             \frac{1}{M} \sum_{m=1}^M p(z_{t+1} \mid z_{0:t}, \theta_{t}^m) \right] \, \exp\left\{-\eta \log \frac{1}{M} \sum_{m=1}^M f(x_{t+1} \mid x_t, \xi_t, \theta_{t}^m)\right\}\to p(z_{t+1} \mid z_{0:t}) \exp\left\{-\eta \log p(x_{t+1} \mid z_{0:t}, \xi_t)\right\} 
        \end{equation}
        almost surely.
        Similarly, using the positivity of $p(z_{t+1} \mid z_{0:t}, \theta_{t}^m)$ and by Lebesgue's dominated convergence theorem, the normalizing constant of the right-hand side of~\eqref{eq:prediction}
        \begin{equation}
                \int \left[ \frac{1}{M} \sum_{m=1}^M p(z_{t+1} \mid z_{0:t}, \theta_{t}^m) \right] \, \exp\left\{-\eta \log \frac{1}{M} \sum_{m=1}^M f(x_{t+1} \mid x_t, \xi_t, \theta_{t}^m)\right\} \, \dd z_{t+1}
        \end{equation}
        converges to 
        \begin{equation}
            \int p(z_{t+1} \mid z_{0:t}) \, \exp\left\{-\eta \log p(x_{t+1} \mid z_{0:t}, \xi_t)\right\} \dd z_{t+1}
        \end{equation}
        and we have
        \begin{equation}
            \frac{\Gamma^M_{t+1}(z_{0:t+1}, \theta_{0:t}^{1:M}, a_{1:t}^{1:M})}{\Gamma^M_{t}(z_{0:t}, \theta_{0:t}^{1:M}, a_{1:t}^{1:M})} 
                \to \frac{p(z_{t+1} \mid z_{0:t}) \exp\left\{-\eta \log p(x_{t+1} \mid z_{0:t}, \xi_t)\right\}}
                         {\int p(z_{t+1} \mid z_{0:t}) \exp\left\{-\eta \log p(x_{t+1} \mid z_{0:t}, \xi_t)\right\} \dd z_{t+1}}.
        \end{equation}
        The recursion is then obtained by noticing that the IBIS step~\eqref{eq:ibis} corresponds to a particle filter update targeting $p(\theta \mid z_{0:t+1})$, so that, under Assumption~\ref{ass:strong-mixing}, we can follow Proposition 11.4 in \citet{chopin2020introduction} to obtain that, for any bounded test function $\psi$, writing
        $$\Gamma^{M,\mathrm{IBIS}}_t(\theta^{1:M}_t) = \frac{\Gamma_{t}^M(z_{0:t}, \theta_{0:t}^{1:M}, a_{1:t}^{1:M})}{\Gamma^M_t(z_{0:t}, \theta_{0:t-1}^{1:M}, a_{1:t-1}^{1:M})},$$ we have
        \begin{equation}
            \mathbb{E}_{\Gamma^{M,\mathrm{IBIS}}_t}\left[\psi(\theta_t)\right] \to \mathbb{E}_{p(\theta_t \mid z_{0:t})}\left[\psi(\theta_t)\right]
        \end{equation}
        almost surely.
        Putting it all together, and noticing that $\Gamma_t(z_{0:t}) p(\theta_t \mid z_{0:t}) = \Gamma_t(z_{0:t}, \theta_t)$, we obtain the result.
    \end{proof}

\section{Algorithmic Details}

\subsection{Reweighting in Inside-Out SMC\textsuperscript{2} for General Potentials}\label{sec:general_weight_function}

\begin{algorithm}[htbp]
    \caption{Reweight function corresponding to~\eqref{eq:stage_reward}}
    \label{alg:reweight-function}
    \begin{algorithmic}[1]  
        \NOTATION Any operation with superscript $m$ is to be understood as performed for all $m = 1, \dots, M$.
        \FUNCTION{\textsc{Reweight}($t$)}
            \STATE $v_t^{mn} = f(x_t^n \mid x_{t-1}^n, \xi_{t-1}^n, \theta_{t-1}^{mn})$.
            \STATE $W_{t, \theta}^{mn} \propto W_{t-1, \theta}^{mn} \, v_t^{mn}$.
            \STATE $r_t^n = \sum_{m=1}^M W_{t, \theta}^{mn} \, \log v_t^{mn} - \log \sum_{m=1}^M W_{t-1, \theta}^{mn} \, v_t^{mn}$.
            \STATE \textbf{return} $g_t^n = \exp\{\eta \, r_t^n\}$.
        \ENDFUNCTION
    \end{algorithmic}
\end{algorithm}

Let us reproduce the general expression for the stage reward at time $t+1$ from~\eqref{eq:stage_reward}.
\begin{align}
         r_t(z_{0:t+1}) &= \mathbb{E}_{p(\theta \mid z_{0:t+1})} \bigl[\log f(x_{t+1}\mid x_t, \xi_t, \theta)\bigr] - \log p(x_{t+1} \mid z_{0:t}, \xi_t) \nonumber \\
         &= \mathbb{E}_{p(\theta \mid z_{0:t+1})} \bigl[\log f(x_{t+1} \mid x_t, \xi_t, \theta)\bigr] - \log \mathbb{E}_{p(\theta \mid z_{0:t})} \left[f(x_{t+1} \mid x_t, \xi_t, \theta)\right]. \label{eq:appendix_stage_reward}
\end{align}
We see that we have expectations with respect to the filtering posteriors of $\theta$ at times $t$ and $t+1$. At line \ref{line:smc2-reweight} of Algorithm~\ref{alg:i-o-smc2}, we have a trajectory $z_{0:t+1}^n$ and a particle representation $\sum_{m=1}^M W_{t, \theta}^{mn} \, \delta_{\theta_t^{mn}}(\theta) \approx p(\theta \mid z_{0:t}^n)$. The second term of the reward function can be estimated as
\begin{equation*}
    \log \mathbb{E}_{p(\theta \mid z_{0:t}^n)} \left[ f(x_{t+1}^n \mid x_t^n, \xi_t^n, \theta) \right] \approx \log \sum_{m=1}^M W_{t, \theta}^{mn} \, f(x_{t+1}^n \mid x_t^n, \xi_t^n, \theta_t^{mn}).
\end{equation*}
Now, to compute the first term, we need to approximate $p(\theta \mid z_{0:t+1}^n)$. For this we perform the reweighting step of IBIS~(line \ref{line:reweight-ibis} from Algorithm~\ref{alg:ibis}) to get updated weights
\begin{align*}
    W_{t+1, \theta}^{mn} &\propto W_{t, \theta}^{mn} \, p_\phi(z_{t+1}^n \mid z_{0:t}^n, \theta_t^{mn}) \\
    &= W_{t, \theta}^{mn} \, f(x_{t+1}^n \mid x_t^n, \xi_t^n, \theta_t^{mn}) \, \pi_\phi(\xi_t^n \mid z_{0:t}^n) \\
    &\propto W_{t, \theta}^{mn} \, f(x_{t+1}^n \mid x_t^n, \xi_t^n, \theta_t^{mn}).
\end{align*}
The distribution $\sum_{m=1}^M W_{t+1, \theta}^{mn} \, \delta_{\theta_t^{mn}}(\theta)$ now approximates the posterior $p(\theta \mid z_{0:t+1}^n)$ as required. We choose not to perform the resample-move step here so as not to introduce additional variance. The first term of~\eqref{eq:appendix_stage_reward} can now be approximated as
\begin{equation*}
    \mathbb{E}_{p(\theta \mid z_{0:t+1}^n)} \bigl[\log f(x_{t+1}^n \mid x_t^n, \xi_t^n, \theta)\bigr] \approx \sum_{m=1}^M W_{t+1, \theta}^{mn} \, \log f(x_{t+1}^n \mid x_t^n, \xi_t^n, \theta_t^{mn}).
\end{equation*}
The entire reweighting procedure is outlined in Algorithm~\ref{alg:reweight-function}.

\subsection{Choice of the Markov Kernel for IBIS}\label{sec:ibis_kernel}

For the Markov kernel $Q_t$ in IBIS~(Algorithm~\ref{alg:ibis}), we follow the choice in \citet{chopin2013smc} and use a Metropolis-Hastings kernel~\citep{metropolis1953equation, hastings1970montecarlo}. If the prior is Gaussian, we use a Gaussian random walk proposal
\begin{equation}\label{eq:random-walk-proposal}
    \tilde{\theta}^m \mid \theta^m \sim \mathcal{N}(\theta^m, c \, \hat{\Sigma}),
\end{equation}
where
\begin{equation}
    \hat{\Sigma} = \frac{1}{\sum_{m=1}^M w^m} \sum_{m=1}^M w^m (\theta^m - \hat{\mu}) (\theta^m - \hat{\mu})^{\top}, \quad \hat{\mu} = \frac{1}{\sum_{m=1}^M w^m} \sum_{m=1}^M w^m \theta^m,
\end{equation}
and $c \in \mathbb{R}_{> 0}$ is a constant that can be tuned to achieve a desired acceptance ratio.
For log-normal priors, we use a similar random walk proposal
\begin{equation}
    \tilde{\theta}^m \mid \theta^m \sim \mathrm{LogNormal}(\theta^m, c \, \hat{\Sigma}).
\end{equation}
We found that the proposal in~\eqref{eq:random-walk-proposal} worked better empirically compared to the proposal $\tilde{\theta}^m \mid \theta^m \sim N(\hat{\mu}, \hat{\Sigma})$ suggested in \citet{chopin2002sequential}, or a proposal which does not use the sample covariance, $\tilde{\theta}^m \mid \theta^m \sim N(\theta^m, cI)$.
In our evaluation, we perform multiple move steps per IBIS step to get a richer representation of samples.

\subsection{Inside-Out SMC\textsuperscript{2} with Conjugate Prior-Likelihood Pairs}\label{sec:conditionally_linear}
\begin{algorithm}[ht]
    \caption{Inside-Out SMC\textsuperscript{2} (Exact)}
    \label{alg:iosmc_exact}
    \begin{algorithmic}[1]
        \NOTATION Any operation with superscript $n$ is to be understood as performed for all $n = 1, \dots, N$.
        \STATE Sample $z_0^n \sim p(\cdot)$.
        \STATE Sample $z_1^n \sim p(\cdot \mid z_0^n)$ and initialize the state history $z_{0:1}^n \leftarrow (z_0^n, z_1^n)$.
        \STATE Compute and normalize the weights $W_z^n \propto g_1(z_{0:1}^n)$.
        \FOR{$t \gets 1, \dots, T - 1$}
            \STATE Sample $b_t^n \sim \mathcal{M}(W_z^{1:N})$.
            \STATE Compute the $\theta$ posterior $p(\theta \mid z_{0:t}^{b_t^n})$.
            \STATE Sample
                $z_{t+1}^n \sim p(\cdot \mid z_{0:t}^{b_t^n})$
                and append to state history
                $z_{0:t+1}^n \gets (z_{0:t}^{b_t^n}, z_{t+1}^n)$.
            \STATE Compute and normalize the weights $W_z^n \propto g_t(z_{0:t}^n)$.
        \ENDFOR
        \STATE \textbf{return} $\{z_{0:T}^n, W_z^n\}_{n=1}^N$.
    \end{algorithmic}
\end{algorithm}

Let $\theta \in \mathbb{R}^{d_\theta}$, $x_t \in \mathbb{R}^{d_x}$ and $\xi_t \in \mathbb{R}^{d_\xi}$. Let us consider conditionally linear, Gaussian transition dynamics for $x$,
\begin{equation}
    f(x_{t+1} \mid x_t, \xi_t, \theta) = \mathcal{N}(x_{t+1} \mid H(x_t, \xi_t)\,\theta, \Sigma(x_t, \xi_t)),
\end{equation}
where $H$ is a map $\mathbb{R}^{d_x + d_\xi} \mapsto \mathbb{R}^{d_x \times d_\theta}$ and $\Sigma:\mathbb{R}^{d_x + d_\xi} \to \mathbb{R}^{d_x \times d_x}$ maps to positive definite matrices. Let us also assume that at time $t$, the filtered posterior of $\theta$ is Gaussian with mean $m_t$ and covariance matrix $P_t$,
\begin{equation}
    p(\theta \mid z_{0:t}) = \mathcal{N}(\theta \mid m_t, P_t).
\end{equation}
Then, using basic identities of the multivariate normal distribution, the marginal density is
\begin{subequations}
    \begin{align}
        p(x_{t+1} \mid z_{0:t}, \xi_t) &= \int f(x_{t+1} \mid x_t, \xi_t, \theta) \, p(\theta \mid z_{0:t}) \, \dd\theta \\
        &= \mathcal{N}(x_{t+1} \mid Hm_t, HP_tH^T + \Sigma),
    \end{align}
\end{subequations}
where the functional dependence of $H$ and $\Sigma$ on $(x_t, \xi_t)$ has been hidden for conciseness. Furthermore, upon observing the next augmented state $z_{t+1} = (x_{t+1}, \xi_t)$, the filtered posterior of $\theta$ can be updated using Bayes' rule:
\begin{align*}
    p(\theta \mid z_{0:t+1}) &= \frac{p_\phi(x_{t+1}, \xi_t \mid z_{0:t}, \theta) \, p(\theta \mid z_{0:t})}{p_\phi(x_{t+1}, \xi_t \mid z_{0:t})} \\
    &= \frac{f(x_{t+1} \mid x_t, \xi_t, \theta) \, \pi_\phi(\xi_t \mid z_{0:t}) \, p(\theta \mid z_{0:t})}{p(x_{t+1} \mid z_{0:t}, \xi_t) \, \pi_\phi(\xi_t \mid z_{0:t})} \\
    &= \frac{f(x_{t+1} \mid x_t, \xi_t, \theta) \, p(\theta \mid z_{0:t})}{p(x_{t+1} \mid z_{0:t}, \xi_t)} \\
    &= \mathcal{N}(\theta \mid m_{t+1}, P_{t+1}),
\end{align*}
where
\begin{equation}
    m_{t+1} = m_t + G (x_{t+1} - Hm_t), \quad P_{t+1} = P_t - G H P_t, \quad G = P_t H^{\top}(H P_t H^{\top} + \Sigma)^{-1}.
\end{equation}
Thus, we see that for a conditionally linear model with Gaussian priors and likelihoods, we can compute the marginal density and the $\theta$ posterior in closed form. The same holds for any conjugate prior-likelihood pair. Consequently, we can compute the stage reward in~\eqref{eq:stage_reward} and hence the potential function in closed form, and the resulting version of the IO-SMC\textsuperscript{2} algorithm that does not use the inner particle filter is given in Algorithm~\ref{alg:iosmc_exact}.

\subsection{Conditional SMC}\label{sec:csmc}
In Section~\ref{sec:target_distribution}, we saw that IO-SMC\textsuperscript{2} is a nested particle filter that targets the distribution $\Gamma_T^M$. In this section, we construct a conditional version of our algorithm that keeps $\Gamma_T^M$ invariant. The basic idea behind CSMC is that given a \emph{reference trajectory} from the target distribution, at each time step in the forward pass, we sample $N-1$ samples conditionally on the reference particle surviving the resampling step~\citep{andrieu2010particle}. 
The CSMC kernel for $\Gamma_t^M$ is outlined in Algorithm~\ref{alg:csmc}, where the potential function estimates $g_t^n$ are computed as detailed in Algorithm~\ref{alg:reweight-function}.
\begin{algorithm}[t]
  \caption{Conditional Inside-Out SMC\textsuperscript{2} kernel}
  \label{alg:csmc}
  \begin{algorithmic}[1]
    \INPUT Reference trajectory $\{z_{0:T}, \{W_{t,\theta}^{\bullet}, \theta_t^{\bullet}\}_{t=0}^{T-1}\}$.
    \OUTPUT New reference trajectory $\{z_{0:T}^*, \{W_{t,\theta}^{\bullet *}, \theta_t^{\bullet *}\}_{t=0}^{T-1}\}$.
    \STATE Set $z_0^1 \gets z_0$, $\theta_0^{\bullet 1} \gets \theta_0^\bullet$ and $W_{0, \theta}^{\bullet 1} \leftarrow W_{0, \theta}^{\bullet}$.
    \FOR{$n = 2, \dots, N$}
        \STATE Sample $z_0^n \sim p(z_0)$, $\theta_0^{\bullet n} \sim p(\theta)$ and set $W_{0, \theta}^{\bullet n} \leftarrow 1/M$.
    \ENDFOR
    \STATE Set $z_{0:1}^1 \leftarrow z_{0:1}$.
    \FOR{$n = 2, \dots, N$}
        \STATE Sample $z_1^n \sim \hat{p}_\phi(\cdot \mid z_0^n)$ and set $z_{0:1}^n \leftarrow (z_0^n, z_1^n)$.
    \ENDFOR
  \STATE Compute and normalize the weights $W_z^n \propto g_1^n$  for all $n = 1, \dots, N$.
    \FOR{$t \gets 1, \dots, T - 1$}
        \STATE Set $z_{t+1}^1 \gets z_{t+1}$, $\theta_t^{\bullet 1} \gets \theta_t^\bullet$, $W_{t,\theta}^{\bullet 1} \gets W_{t,\theta}^{\bullet}$, and $z_{0:t+1}^1 \gets z_{0:t+1}$.
        \FOR{$n = 2, \dots, N$}
            \STATE Sample $b_t^n \sim \mathcal{M}(W_z^{1:N})$.
            \STATE $\theta_t^{\bullet n}, W^{\bullet n}_{t, \theta}, \gets \textsc{Ibis\_Step}(z_{0:t}^{b_t^n}, \theta_{t-1}^{\bullet b_t^n}, W^{\bullet b_t^n}_{t - 1, \theta})$
            \STATE Sample
                $z_{t+1}^n \sim \hat{p}_\phi(\cdot \mid z_{0:t}^{b_t^n}),$
                and append to state history
                $z_{0:t+1}^n \gets [z_{0:t}^{b_t^n}, z_{t+1}^n]$.
        \ENDFOR
        \STATE Compute and normalize the weights $W_z^n \propto g_t^n$ for all $n = 1, \dots, N$.
    \ENDFOR
    \STATE Sample an index $j \sim \mathcal{M}(W_z^{1:N})$.
    \STATE \textbf{return} $\{z_{0:T}^j, \{W_{t,\theta}^{\bullet j}, \theta_t^{\bullet j}\}_{t=0}^{T-1}\}$.
  \end{algorithmic}
\end{algorithm}

While Algorithm~\ref{alg:csmc} may look more complicated than ``classical'' CSMC algorithms~\citep{andrieu2010particle}, its complexity may be abstracted away by remembering that a `particle' object is, in our case, an instance of the approximate inner distribution $\Gamma_t^M$, which is associated with its weights, particles and ancestors, noting that the ancestors do not appear in the computation of Algorithm~\ref{alg:reweight-function} and are therefore omitted.

\section{Experimental Details}\label{sec:exp_details}
\subsection{Network Architectures and Hyperparameters}
We use the same network architecture for all amortized policies in the evaluation. The architecture of our design policy network is similar to that in iDAD~\citep{ivanova2021implicit}, with an encoder network transforming the augmented state sequences into a stacked representation $\{R(z_s)\}_{s=0}^t$ before passing it to the recurrent layers. The encoder networks for the augmented states are given in Table~\ref{tab:policy-arch} and Table~\ref{tab:lstm-arch}. For training, we used the Adam optimizer~\citep{kingma2014adam}. Our policies were trained on a single 9\textsuperscript{th} Gen Intel Core i9 processor, while iDAD policies were trained on an Nvidia A100 GPU using the authors' implementation~\citep{ivanova2021implicit}. The hyperparameters used to train iDAD are listed in Table~\ref{tab:hyperparams-idad}, and are common to all experiments.

All IO-SMC\textsuperscript{2} policies were trained with an additional slew rate penalty on the designs. We noticed that this detail promoted smoother design trajectories, that facilitated the amortization of the recurrent policy.

\newpage

\begin{table}[t]
    \caption{The encoder architecture.}
    \vspace{1em}
    \centering
    \begin{tabular}{cccc}
        \toprule
        Layer & Description & Size & Activation \\
        \midrule
        Input & Augmented state $z$ & dim($z$) & - \\
        Hidden layer 1 & Dense & 256 & ReLU \\
        Hidden layer 2 & Dense & 256 & ReLU \\
        Output & Dense & 64 & - \\
        \bottomrule
    \end{tabular}
    \label{tab:policy-arch}
    \vskip -0.05in
\end{table}
\begin{table}[h]
    \caption{The recurrent network architecture.}
    \vspace{1em}
    \centering
    \begin{tabular}{cccc}
        \toprule
        Layer & Description & Size & Activation \\
        \midrule
        Input & $\{R(z_s)\}_{s=0}^t$ & $64 \cdot (t+1)$ & - \\
        Hidden layer 1 & LSTM / GRU & 64 & - \\
        Hidden layer 2 & LSTM / GRU & 64 & - \\
        Hidden layer 3 & Dense & 256 & ReLU \\
        Hidden layer 4 & Dense & 256 & ReLU \\
        Output & Designs $\xi$ & dim($\xi$) & - \\
        \bottomrule
    \end{tabular}
    \label{tab:lstm-arch}
    \vskip -0.15in
\end{table}
\begin{table}[h]
    \caption{Hyperparameters for iDAD.}
    \label{tab:hyperparams-idad}
    \vspace{1em}
    \centering
    \begin{tabular}{lc}
        \toprule
        Hyperparameter & iDAD \\
        \midrule
        Batch size & $512$ \\
        Number of contrastive samples & $16383$ \\
        Number of gradient steps & $10000$ \\
        Learning rate (LR) & $5 \times 10^{-4}$ \\
        LR annealing parameter & $0.96$ \\
        LR annealing frequency (if no improvement) & $400$ \\
        \bottomrule
    \end{tabular}
\end{table}

\subsection{Stochastic Pendulum Experiment}
We consider two different dynamics for the compound pendulum, one conditionally linear in the parameters and another that is fully nonlinear. The following specifications are similar in both settings. 

Let $x_t = [q_t, \dot{q}_t]^{\top}$ denote the state of the pendulum, with $q_t$ being the angle from the vertical and $\dot{q}_t$ the angular velocity. The parameters of interest are $(m, l)$, the mass and length of the pendulum, while $g=9.81$ and $d=0.1$ are the gravitational acceleration and damping constants. The design, $\xi_t \in [-1, 1]$, is the torque applied to the pendulum. We discretize the respective SDEs that describe the dynamics using Euler-Maruyama with a step size $\dd t = 0.05$ and consider a horizon of $T = 50$ experiments. The initial state is fixed to $x_0 = [0, 0]^{\top}$.

\subsubsection{Conditionally Linear formulation}\label{app:pendulum-linear}
In this setting, we transform the non-linear pendulum equations to obtain a conditional linear dependency on the parameters, similar to \citet{belusov2019belief}. The resulting parameter vector is $\theta = \displaystyle \left[\frac{3g}{2l}, \frac{3d}{ml^2}, \frac{3}{ml^2} \right]^{\top}$. The dynamics is described by the following Ito SDE~\citep[][Chapter 3]{sarkka2019applied}
\begin{equation}
    \dd x_t = h(x_t, \xi_t)^{\top} \, \theta \, \dd t + L \, \dd \beta,
\end{equation}
with a drift term $h(x_t, \xi_t) = [-\sin(q), -\dot{q}, \xi_t]^{\top}$, diffusion term $L = [0, 0.1]^{\top}$ and Brownian motion $\beta$.

To maintain conjugacy, we assume a Gaussian prior 
\begin{equation*}
    p(\theta) = \textrm{Normal}\left( \begin{bmatrix}
        14.7 \\ 0 \\ 3.0
    \end{bmatrix}, \begin{bmatrix}
        0.1 & 0 & 0 \\
        0 & 0.01 & 0 \\
        0 & 0 & 0.1
    \end{bmatrix} \right).
\end{equation*}
The remaining hyperparameters are listed in Table~\ref{tab:hyperparams-pendulum-linear}.
\begin{table}[h!]
    \caption{Hyperparameters for the conditionally linear pendulum experiment.}
    \label{tab:hyperparams-pendulum-linear}
    \vspace{1em}
    \centering
    \begin{tabular}{ccc}
        \toprule
        Hyperparameter & IO-SMC\textsuperscript{2} & IO-SMC\textsuperscript{2} (Exact) \\
        \midrule
        N & $256$ & $256$ \\
        M & $128$ & - \\
        Tempering $(\eta)$ & $1.0$ & $1.0$ \\
        Slew rate penalty & $0.1$ & $0.1$ \\
        IBIS moves & $3$ & - \\
        Learning rate & $10^{-3}$ & $10^{-3}$ \\
        Training iterations & $25$ & $25$ \\
        \bottomrule
    \end{tabular}
\end{table}

\subsubsection{Nonlinear formulation}\label{app:pendulum-non-linear}
The unknown parameters are $\theta = (m, l)$. The dynamics is described by the SDE
\begin{equation*}
    \dd x_t = h(x_t, \xi_t, \theta)^{\top} \, \dd t + L \, \dd \beta,
\end{equation*}
where $h(x_t, \xi_t, \theta) = [\dot{q}_t, \Ddot{q}_t]^{\top}$ and
\begin{equation*}
    \Ddot{q}_t = -\frac{3g}{2l}\sin{q_t} + \frac{(\xi_t - d \, \dot{q}_t)}{ml^2},
\end{equation*}
and $L = [0, 0.1]^{\top}$. The prior is a log-normal distribution
\begin{equation*}
    p(\theta) = \textrm{LogNormal}\left( \begin{bmatrix}
        0 \\ 0
    \end{bmatrix}, \begin{bmatrix}
        0.01 & 0 \\ 0 & 0.01
    \end{bmatrix} \right).
\end{equation*}
The remaining hyperparameters for IO-SMC\textsuperscript{2} are listed in Table~\ref{tab:exp-hyperparams}, and the training progression is depicted in Figure~\ref{fig:nonlinear-pendulum-eig-training}.
\vspace{-1em}
\begin{table}[ht]
    \caption{Hyperparameters for the non-linear pendulum, stochastic cart-pole, and stochastic dual-link experiments.}
    \label{tab:exp-hyperparams}
    \vspace{1em}
    \centering
    \begin{tabular}{cccc}
        \toprule
        Hyperparameter & Nonlinear pendulum & Stochastic cart-pole & Stochastic dual-link \\
        \midrule
        N & $256$ & $256$ & $256$ \\
        M & $128$ & $1024$ & $128$  \\
        Tempering ($\eta$) & $1.0$ & $0.25$ & $0.25$  \\
        Slew rate penalty & $0.2$ & $0.1$ & $0.1$  \\
        IBIS moves & $3$ & $3$ & $3$  \\
        Learning rate & $5 \times 10^{-4}$ & $5 \times 10^{-4}$ & $5 \times 10^{-4}$ \\
        Training iterations & $25$ & $15$ & $25$  \\
        \bottomrule
    \end{tabular}
\end{table}
\vspace{-0.5em}

\subsection{Stochastic Cart-Pole Experiment}\label{app:cart-pole}
The cart-pole is described by a state $x_t = [s_t, q_t, \dot{s}_t, \dot{q}_t]^{\top}$, where $s_t$ and $\dot{s}_t$ are the position and velocity of the cart, and $q_t$ and $\dot{q}_t$ are the position and velocity of the pole. The unknowns are $\theta = (l, m_p, m_c)$, the length and mass of the pendulum and the mass of the cart, respectively. The design, $\xi_t \in [-5, 5]$, is the force applied to the cart. The corresponding SDE is
\begin{equation*}
    \dd x_t = h(x_t, \xi_t, \theta) \, \dd t + L \, \dd \beta,
\end{equation*}
where $h(x_t, \xi_t, \theta) = [\dot{s}_t, \dot{q}_t, \Ddot{s}_t, \Ddot{q}_t]^{\top}$ and 
\begin{align*}
    \Ddot{s}_t &= \frac{1}{m_c + m_p \sin^2{\!q_t}} \bigl[ \xi_t + m_p \sin{q_t} (l \dot{q}_t^2 + g\cos{q_t}) - (k_1 s_t + d_1 \dot{s_t}) - (k_2 q_t + d_2 \dot{q_t}) \cos{q_t} / l \bigr], \\
    \Ddot{q}_t &= \begin{multlined}
        \frac{1}{l(m_c + m_p \sin^2{\!q_t})} \bigl[ -\xi_t \cos{q_t} - m_p l \dot{q}_t^2 \cos{q_t}\sin{q_t} - (m_c + m_p)g\sin{q_t} \\ - (k_1 s_t + d_1 \dot{s_t}) \cos{q_t} - (k_2 q_t + d_2 \dot{q_t}) \cos^2{\!q_t} / l \bigr].
    \end{multlined}
\end{align*}
$(k_1, k_2, d_1, d_2) = 0.01$ are the linear and torsional stiffness and damping constants and $g=9.81$. The diffusion term is $L = [0, 0, 0.1, 0]^{\top}$. We discretize the SDE with a step size $\dd t = 0.05$ and consider a horizon of $T = 50$ experiments. The initial state is fixed at $x_0 = [0, 0, 0, 0]^{\top}$. The prior for $\theta$ is log-normal
\begin{equation*}
    p(\theta) = \textrm{LogNormal}\left( \begin{bmatrix}
        0 \\ 0 \\ 0
    \end{bmatrix}, \begin{bmatrix}
        0.01 & 0 & 0 \\ 0 & 0.01 & 0 \\ 0 & 0 & 0.01
    \end{bmatrix} \right).
\end{equation*}
The remaining hyperparameters used for this experiment are given in Table~\ref{tab:exp-hyperparams}. In Table~\ref{tab:eig-spce-comparison}, we compare the sample efficiency of our method against the sPCE bound as an estimator of the EIG. For a trained policy and a fixed number of trajectory samples, we vary the number of parameter particles~(contrastive samples). Our method reaches the asymptotic value of approximately $21$ at $1024$ particles, while the sPCE does not and fully saturates its sample-size dependent upper bound~\citep[Appendix A]{foster2021deep} for all the experiments we ran. Finally, the training progression of the algorithm is depicted in Figure~\ref{fig:cart-pole-eig-training}.
\begin{table}[htbp]
    \caption{EIG estimates and sPCE lower bounds for a trained policy on the stochastic cart-pole experiment for different numbers of $\theta$ particles. We report the mean and standard deviation over 25 random seeds.}
    \label{tab:eig-spce-comparison}
    \vskip 0.25cm
    \centering
    \begin{tabular}{cccc}
        \toprule
        No. of $\theta$ particles ($M$) & EIG Estimate~\eqref{eq:eig-estimator} & sPCE & sPCE theoretical limit~($\log{(M + 1})$) \\
        \midrule
        $64$   & $14.52 \pm 0.71$ & $4.17 \pm 0.00$ & $4.17$ \\
        $128$  & $17.38 \pm 0.93$ & $4.86 \pm 0.00$ & $4.86$ \\
        $256$  & $19.20 \pm 0.68$ & $5.55 \pm 0.00$ & $5.55$ \\
        $512$  & $20.67 \pm 0.82$ & $6.24 \pm 0.00$ & $6.24$ \\
        $1024$ & $21.29 \pm 0.63$ & $6.93 \pm 0.00$ & $6.93$ \\
        $2048$ & $21.60 \pm 0.37$ & $7.62 \pm 0.00$ & $7.62$ \\
        $4096$ & $21.87 \pm 0.56$ & $8.32 \pm 0.00$ & $8.32$ \\
        $8192$ & $21.77 \pm 0.56$ & $9.01 \pm 0.00$ & $9.01$ \\
        \bottomrule
    \end{tabular}
\end{table}

\subsection{Stochastic Double-Link Experiment}\label{app:double-pendulum}
The double-link system is described by a state vector $x = [q, \dot{q}]^{\top} = [q_1, q_2, \dot{q}_1, \dot{q}_2]^{\top}$, where $q_1$ and $q_2$ are the angles of the two joints~\citep[Appendix B]{underactuated}. The system is doubly actuated, with $\xi = [\xi_1, \xi_2]^\top$, $\xi_1 \in [-4, 4]$ and $\xi_2 \in [-2, 2]$. The parameters of interest are $\theta = (m_1, m_2, l_1, l_2)$, the masses and lengths of the two joints, respectively. We define
\begin{align}
    M(q) &= \begin{bmatrix}
        (m_1 + m_2) l_1^2 + m_2 l_2^2 + 2 m_2 l_1 l_2 \cos{q}_2 & m_2 l_2^2 + m_2 l_1 l_2 \cos{q}_2 \\
        m_2 l_2^2 + m_2 l_1 l_2 \cos{q}_2 & m_2 l_2^2
    \end{bmatrix}, \\
    C(q, \dot{q}) &= \begin{bmatrix}
        0 & -m_2 l_1 l_2 (2 \dot{q}_1 + \dot{q}_2) \sin{q}_2 \\
        \frac{1}{2}m_2 l_1 l_2 (2 \dot{q}_1 + \dot{q}_2) \sin{q}_2 & -\frac{1}{2} m_2 l_1 l_2 \dot{q}_1 \sin{q}_2
    \end{bmatrix}, \\
    \tau_g(q) &= -g \cdot \begin{bmatrix}
        (m_1 + m_2) l_1 \sin(q_1) + m_2 l_2 \sin(q_1 + q_2) \\ m_2 l_2 \sin(q_1 + q_2)
    \end{bmatrix},
\end{align}
where $g=9.81$. The system dynamics are then given by the equations
\begin{align}
    \dd q &= \dot{q} \, \dd t, \\
    \dd \dot{q} &= M^{-1}(q) [\tau_g(q) + \xi - C(q, \dot{q}) \dot{q}] \, \dd t + L \dd \beta,
\end{align}
where $L = \begin{bmatrix}
    0.1 & 0 \\ 0 & 0.1
\end{bmatrix}$
and $\beta = [\beta_1, \beta_2]^{\top}$, with $\beta_1$ and $\beta_2$ being two independent Brownian motions. We assume a log-normal prior of the form
\begin{equation*}
    p(\theta) = \textrm{LogNormal}\left( \begin{bmatrix}
        0 \\ 0 \\ 0 \\ 0
    \end{bmatrix}, \begin{bmatrix}
        0.01 & 0 & 0 & 0 \\
        0 & 0.01 & 0 & 0 \\
        0 & 0 & 0.01 & 0 \\
        0 & 0 & 0 & 0.01
    \end{bmatrix} \right).
\end{equation*}
The remaining hyperparameters for IO-SMC\textsuperscript{2} are listed in Table~\ref{tab:exp-hyperparams}, and the training progression is depicted in Figure~\ref{fig:double-pendulum-eig-training}.

\newpage
\begin{figure}[t]
    \centering
    \vspace{-0.2cm}
    \begin{tikzpicture}

\definecolor{forestgreen4416044}{RGB}{0,77,64}
\definecolor{steelblue31119180}{RGB}{30,136,229}

\begin{axis}[
    width=7cm,
    height=6cm,
    grid=both,
    minor tick num=3,
    grid style={line width=.1pt, draw=gray!10},
    major grid style={line width=.1pt, draw=gray!50},
    xlabel=Training Epochs,
    ylabel=EIG Estimate,
    legend style={
        nodes={scale=0.85, transform shape},
        at={(0.95,0.05)},
        anchor=south east
    },
    legend cell align={left},
    ymin = -0.25,
    ymax = 4.0
]
    \addplot [
        black,
        thick,
        mark=*,
        mark size=2,
        error bars/.cd,
            y dir=both,y explicit,
    ] coordinates {
        (1.0, 0.15314773344226512) +- (0.0, 0.019699377515776374)
        (2.0, 0.22331419967381122) +- (0.0, 0.059576540412495826)
        (3.0, 0.4454209332314236) +- (0.0, 0.2052344545405983)
        (4.0, 0.7698848094853378) +- (0.0, 0.33996818563743)
        (5.0, 1.1779555255161231) +- (0.0, 0.5216829148697699)
        (6.0, 1.5934098495480595) +- (0.0, 0.5799072330224897)
        (7.0, 1.978225932407009) +- (0.0, 0.5587235870357407)
        (8.0, 2.338701129779168) +- (0.0, 0.4715956114462798)
        (9.0, 2.6712102111211764) +- (0.0, 0.4022408077224845)
        (10.0, 2.911200954005488) +- (0.0, 0.35340984152487226)
        (11.0, 3.079623530668574) +- (0.0, 0.2934770977027224)
        (12.0, 3.1950154414749408) +- (0.0, 0.25218359588291683)
        (13.0, 3.288476973530609) +- (0.0, 0.21793224545436993)
        (14.0, 3.3470887053662555) +- (0.0, 0.1988165368952456)
        (15.0, 3.4126056057926153) +- (0.0, 0.1576931569387095)
        (16.0, 3.4959092821154507) +- (0.0, 0.14876153992238728)
        (17.0, 3.5107310815212363) +- (0.0, 0.10192916485247716)
        (18.0, 3.5441931530960726) +- (0.0, 0.09457033048998399)
        (19.0, 3.5605734787760586) +- (0.0, 0.09463049895508163)
        (20.0, 3.5828554823322314) +- (0.0, 0.08106451159487116)
        (21.0, 3.608791527138221) +- (0.0, 0.07024261268499041)
        (22.0, 3.6061787234318485) +- (0.0, 0.06699892012705205)
        (23.0, 3.604602051485013) +- (0.0, 0.07256547486622694)
        (24.0, 3.622635997693572) +- (0.0, 0.08897617213317162)
        (25.0, 3.640738972433983) +- (0.0, 0.06565269742366543)
    };
    \addlegendentry{IO-SMC\textsuperscript{2}}
\end{axis}
\end{tikzpicture} 
    \vspace{-0.25cm}
    \caption{Training progression of the IO-SMC\textsuperscript{2} policy on the non-linear stochastic pendulum experiment. At every epoch, we evaluate the EIG estimate using the mean policy. We report the mean and standard deviation of the EIG estimate over $25$ unique training seeds.}
    \label{fig:nonlinear-pendulum-eig-training}
    \vspace{-0.25cm}
\end{figure}
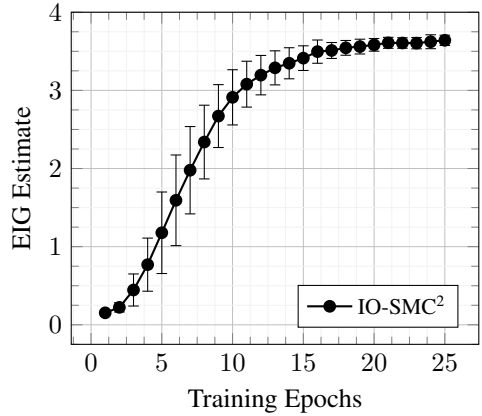
\begin{figure}[t]
    \centering
    \begin{tikzpicture}

\begin{axis}[
    width=7cm,
    height=6cm,
    grid=both,
    minor tick num=3,
    grid style={line width=.1pt, draw=gray!10},
    major grid style={line width=.1pt, draw=gray!50},
    xlabel=Training Epochs,
    ylabel=EIG Estimate,
    legend style={
        nodes={scale=0.85, transform shape},
        at={(0.95,0.05)},
        anchor=south east
    },
    legend cell align={left},
]
    \addplot [
        black,
        thick,
        mark=*,
        mark size=2,
        error bars/.cd,
            y dir=both,y explicit,
    ] coordinates {
        (1.0, 7.491134781701749) +- (0.0, 1.2368362903372183)
        (2.0, 9.512416877784007) +- (0.0, 2.3463653821749886)
        (3.0, 12.697692129825386) +- (0.0, 3.1620643020593224)
        (4.0, 15.84311275037923) +- (0.0, 2.8025896873333997)
        (5.0, 17.771242749533904) +- (0.0, 1.991397885022261)
        (6.0, 19.020893035436305) +- (0.0, 1.311400427022583)
        (7.0, 19.66628501344076) +- (0.0, 0.9111540995317713)
        (8.0, 20.01973735785842) +- (0.0, 0.8331555698519646)
        (9.0, 20.126336042248937) +- (0.0, 0.8493262225472693)
        (10.0, 20.11459951025506) +- (0.0, 0.7894778586793879)
        (11.0, 20.116340221270725) +- (0.0, 0.8256668530318919)
        (12.0, 20.199440080614927) +- (0.0, 0.7949264684337433)
        (13.0, 20.176606031266587) +- (0.0, 0.795454728168891)
        (14.0, 20.12648820288104) +- (0.0, 0.8131109520238227)
        (15.0, 20.158403354286143) +- (0.0, 0.7932433343287377)
    };
    \addlegendentry{IO-SMC\textsuperscript{2}}
\end{axis}
\end{tikzpicture}
    \vspace{-0.25cm}
    \caption{Training progression of the IO-SMC\textsuperscript{2} policy for the stochastic cart-pole experiment. At every epoch, we evaluate the EIG estimate using the mean policy. We report the mean and standard deviation of the EIG estimate over $25$ unique training seeds.}
    \label{fig:cart-pole-eig-training}
    \vspace{-0.25cm}
\end{figure}
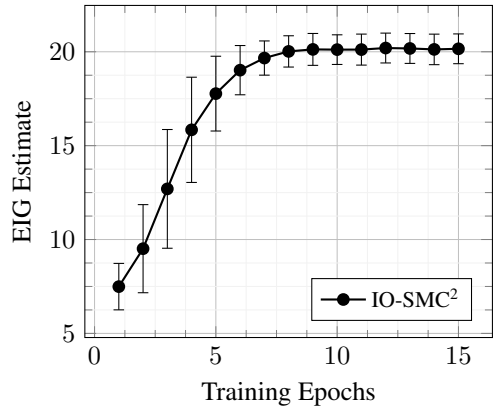
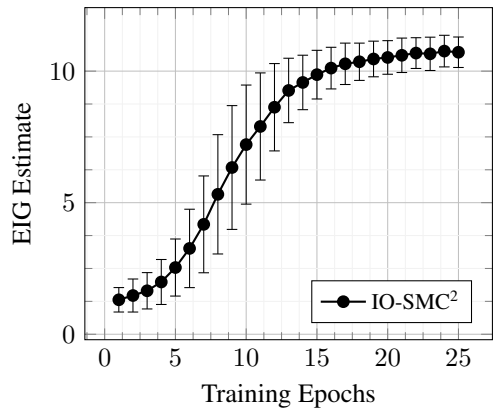
\begin{figure}[t]
    \centering
    \begin{tikzpicture}

\begin{axis}[
    width=7cm,
    height=6cm,
    grid=both,
    minor tick num=3,
    grid style={line width=.1pt, draw=gray!10},
    major grid style={line width=.1pt, draw=gray!50},
    xlabel=Training Epochs,
    ylabel=EIG Estimate,
    legend style={
        nodes={scale=0.85, transform shape},
        at={(0.95,0.05)},
        anchor=south east
    },
    legend cell align={left},
]
    \addplot [
        black,
        thick,
        mark=*,
        mark size=2,
        error bars/.cd,
            y dir=both,y explicit,
    ] coordinates {
        (1.0,  1.3091088779962) +- (0.0, 0.4644504900984)
        (2.0,  1.4716060863128) +- (0.0, 0.6276555721804)
        (3.0,  1.6517638685160) +- (0.0, 0.6905817832728)
        (4.0,  1.9859028988859) +- (0.0, 0.8547289312624)
        (5.0,  2.5351828160088) +- (0.0, 1.0845042114358)
        (6.0,  3.2614604985451) +- (0.0, 1.4895967282974)
        (7.0,  4.1762813636960) +- (0.0, 1.8400856093791)
        (8.0,  5.3177545577495) +- (0.0, 2.2669994307635)
        (9.0,  6.3367156191676) +- (0.0, 2.3523765815343)
        (10.0, 7.2075733589245) +- (0.0, 2.2628284046999)
        (11.0, 7.8964588939329) +- (0.0, 2.0357121164272)
        (12.0, 8.6271793421528) +- (0.0, 1.6596768003818)
        (13.0, 9.2652227252071) +- (0.0, 1.2242745362892)
        (14.0, 9.5692777658673) +- (0.0, 1.0350245166163)
        (15.0, 9.8657382261507) +- (0.0, 0.9248310846987)
        (16.0, 10.114530978423) +- (0.0, 0.7911348586235)
        (17.0, 10.277726379508) +- (0.0, 0.7865548104753)
        (18.0, 10.355601217999) +- (0.0, 0.7084997054798)
        (19.0, 10.461403379979) +- (0.0, 0.6772925259732)
        (20.0, 10.519360346522) +- (0.0, 0.6383741500029)
        (21.0, 10.601281706779) +- (0.0, 0.6541958949355)
        (22.0, 10.685347617715) +- (0.0, 0.5843175244360)
        (23.0, 10.656776244029) +- (0.0, 0.6332142981074)
        (24.0, 10.760820446228) +- (0.0, 0.6007620506667)
        (25.0, 10.717302065123) +- (0.0, 0.5790462471762)
    };
    \addlegendentry{IO-SMC\textsuperscript{2}}
\end{axis}
\end{tikzpicture}
    \vspace{-0.25cm}
    \caption{Training progression of the IO-SMC\textsuperscript{2} policy for the stochastic double-link experiment. At every epoch, we evaluate the EIG estimate using the mean policy. We report the mean and standard deviation of the EIG estimate over $25$ unique training seeds.}
    \label{fig:double-pendulum-eig-training}
    \vspace{-0.25cm}
\end{figure}

\end{document}